%% file: main.tex
\documentclass{article}

\usepackage[main, final]{neurips_2026}

\usepackage[utf8]{inputenc} 
\usepackage[T1]{fontenc}    
\usepackage{float}
\usepackage{hyperref}       
\usepackage{url}            
\usepackage{booktabs}       
\usepackage{amsfonts}       
\usepackage{mathtools}
\usepackage{nicefrac}       
\usepackage{microtype}      
\usepackage{xcolor}         
\usepackage{graphicx}
\usepackage{adjustbox}
\usepackage{multirow}
\usepackage[strings]{underscore}
\usepackage{makecell}
\usepackage{threeparttable}

\usepackage{amsmath}
\usepackage[nameinlink,capitalize]{cleveref}

\usepackage{amsthm}
\newtheorem{theorem}{Theorem}[section]

\newtheorem{proposition}[theorem]{Proposition}
\newtheorem{assumption}{Assumption}
\newtheorem{lemma}[theorem]{Lemma}

\usepackage{algorithm}
\usepackage{algorithmic}
\usepackage{fontawesome5}

\usepackage{booktabs}
\usepackage{makecell} 
\usepackage{threeparttable}

\title{Denoising Implicit Variational Inference}

\author{%
    Zitong Wang\thanks{
    Equal contribution.
    } \quad
    Zekun Wu\footnotemark[1] \quad
    Longlin Yu \quad
    Cheng Zhang\thanks{Correspondence to Cheng Zhang \texttt{<chengzhang@math.pku.edu.cn>}.}
    \\[0.5em]
    School of Mathematical Sciences\\
    Peking University
}

\begin{document}

\maketitle

\begin{abstract}

Semi-implicit variational inference (SIVI) expands the representational capacity of variational families, but optimizing the Evidence Lower Bound (ELBO) remains challenging because the marginal score of the semi-implicit distribution is generally intractable.
Unbiased Implicit Variational Inference (UIVI) and related methods address this difficulty through reparameterized ELBO gradients, but existing approaches estimate the required score indirectly through reverse-conditional sampling or approximation, which can be computationally costly and unstable.
We propose Denoising Implicit Variational Inference (DIVI), which learns the marginal score directly via denoising score matching.
The learned score is used as a plug-in estimator in the pathwise ELBO gradient, replacing inner reverse-conditional sampling with a score-network evaluation.
We analyze DIVI as inexact stochastic gradient ascent and show, under stated assumptions, that the averaged stationarity measure is controlled by the usual optimization term and the average score-matching error.
Empirically, we evaluate DIVI on both synthetic distributions and a variety of real-data Bayesian inference tasks. The results show that DIVI improves over the current UIVI baselines while reducing the cost of ELBO-gradient estimation.
Code is available at \faGithub\hspace{0.4em}\url{https://github.com/VincentWangzt/DIVI}.
\end{abstract}

\section{Introduction}
\label{sec:intro}

Bayesian inference provides a principled framework for uncertainty quantification, but exact posterior inference is often intractable in modern probabilistic models. Variational Inference (VI) addresses this difficulty by casting posterior approximation as an optimization problem: given a variational family $\mathcal{Q}$, VI seeks a tractable approximation $q \in \mathcal{Q}$ by maximizing the Evidence Lower Bound (ELBO), equivalently minimizing a KL divergence up to an additive constant~\citep{Jordan1999Introduction, zhang2018advances} .
The choice of $\mathcal{Q}$ determines the tradeoff between computational tractability and approximation quality.
Classical mean-field VI is computationally convenient due to its factorized structure, but is often too restrictive to capture posterior dependencies and complex geometry~\citep{peterson1987mean, saul1996mean,turner2011two, blei2017variational}.

Subsequent work has relaxed these restrictions through black-box gradient estimators and richer variational families, including auxiliary-variable models, structured approximations, and normalizing flows, substantially improving the expressiveness of VI~\citep{Ranganath2014Black,mnih2014neural,paisley2012variational,giordano2015linear,kingma2016improved,dinh2016density}.
However, many such methods still require evaluating and differentiating the variational density, which restricts the admissible variational family.

Implicit VI relaxes this restriction by defining variational distributions through sampling procedures, such as neural generators, without requiring tractable densities~\citep{Tran2017HierarchicalIM, Huszr2017VariationalIU, Shi18kivi}.
The central obstacle then shifts to ELBO optimization, since marginal log-density and marginal score are no longer directly available.
Semi-implicit variational inference (SIVI)~\citep{yin2018semi, moens2021efficientsivi} offers a compromise by combining an explicit conditional distribution with an implicit mixing distribution.
The resulting variational family is easy to sample from and highly expressive, but exact ELBO gradients require the marginal score $\nabla_z \log q_\phi(z)$, which is intractable because it involves integrating over the mixing variable.

Existing ELBO-based SIVI methods address this marginal-score bottleneck indirectly.
Classical SIVI optimizes biased surrogate ELBOs, whereas unbiased implicit variational inference (UIVI)~\citep{titsias2019unbiased} derives an unbiased gradient estimator of the exact ELBO assuming access to samples from the reverse conditional distribution over the mixing variable. In practice, this reverse conditional is approximated by inner-loop Markov chain Monte Carlo (MCMC) sampling, which can be expensive and unstable when the reverse conditional is high-dimensional or multimodal.
Importance-weighted variants reduce the reliance on MCMC by using importance sampling from proposal distributions~\citep{Sobolev2019ImportanceWH}, and recent AISIVI-style methods improve the proposal through conditional normalizing flows (CNFs)~\citep{pielok2025revisiting}.
However, these approaches still depend on an inner approximation to the reverse conditional, resulting in difficulties to control the variance of the gradient estimator in practice. The sampling-based approximation techniques also add computational overhead to the optimization procedure.

In this paper, we propose Denoising Implicit Variational Inference (DIVI), which estimates the marginal score of a semi-implicit variational distribution directly through denoising score matching (DSM)~\citep{vincent2011aconnectionbetweenscorematching}. The key observation is that samples from the semi-implicit hierarchy provide pairs $(z,\epsilon)$ for which the conditional score $\nabla_z\log q_\phi(z|\epsilon)$ is tractable. Instead of sampling from an approximate reverse conditional, DIVI trains a score network to regress the marginal score from these conditional-score targets.
The learned score is then used as a plug-in estimator in the pathwise ELBO gradient, replacing inner proposal-based approximation with a single score-network evaluation.


This formulation yields both an efficient algorithm and a tractable optimization analysis.
Our main contributions are summarized as follows:
\begin{itemize}
    \item \textbf{Direct marginal-score learning for semi-implicit VI.}
    We propose DIVI, which learns the marginal score of a semi-implicit variational distribution via denoising score matching and uses the learned score as a plug-in estimator in the pathwise ELBO gradient, replacing inner reverse-conditional sampling with a single score-network evaluation.
    \item \textbf{Convergence analysis with inexact score estimates.}
    We analyze DIVI through the lens of inexact stochastic gradient ascent and show, under standard smoothness and boundedness assumptions, that the averaged stationarity measure is controlled by the usual optimization error and the average score-matching error.
    \item \textbf{Empirical validation on synthetic distribution and real-data Bayesian inference tasks.}
    Across these settings, DIVI recovers challenging posterior structure, matches or improves upon current baselines, and achieves favorable wall-clock efficiency in high-dimensional diffusion and Bayesian neural-network experiments.
\end{itemize}

\section{Background}
\subsection{Semi-implicit Variational Family}
SIVI~\citep{yin2018semi} posits a flexible family that mixes an explicit tractable conditional layer $q_\phi(z| \epsilon)$ with a simple mixing distribution $q(\epsilon)$:
$$
\epsilon \sim q(\epsilon), \qquad z \sim q_\phi(z| \epsilon), \qquad
q_\phi(z) = \int q_\phi(z| \epsilon)\, q(\epsilon)\,\mathrm{d}\epsilon.
$$
This construction makes $q_\phi(z)$ semi-implicit: sampling is straightforward via drawing from $q(\epsilon)$ followed by $q_\phi(z| \epsilon)$, while evaluating the marginal density $q_\phi(z)$ is generally intractable due to the integral. This intractability of $\log q_\phi(z)$ makes optimizing the exact ELBO challenging, and SIVI employs tractable surrogate objectives that become tight asymptotically.
We write $s_\phi(z):=\nabla_z\log q_\phi(z)$ for the marginal score and $s_\phi(z|\epsilon):=\nabla_z\log q_\phi(z|\epsilon)$ for the conditional score.

\subsection{Unbiased Implicit Variational Inference}
Instead of optimizing surrogate ELBOs for semi-implicit variational families, 
UIVI targets the exact ELBO and derives an unbiased gradient estimator~\citep{titsias2019unbiased}. We follow the setup in UIVI, adopting the mild assumption of reparameterizability: there exists a base noise distribution $q(u)$ and a deterministic map $h_\phi$ such that
\begin{equation}
u\sim q(u),\qquad z=h_\phi(u;\epsilon)
\quad \Longleftrightarrow\quad
z\sim q_\phi(z| \epsilon).
\label{eq:uivi-reparam}
\end{equation}
Denote the dataset as $\mathcal{D}$. With the reparameterization $z=h_\phi(u;\epsilon)$, the exact ELBO can be rewritten as
\begin{align}
\mathcal L(\phi)
&= \mathbb E_{q_\phi(z)}\big[\log p(\mathcal{D},z)-\log q_\phi(z)\big] \nonumber\\
&= \mathbb E_{q(\epsilon)\,q(u)}
\Big[\log p(\mathcal{D},z)-\log q_\phi(z)\Big]\Big|_{z=h_\phi(u;\epsilon)}.
\label{eq:elbo-rewrite}
\end{align}
Pushing $\nabla_\phi$ inside the expectation yields the following decomposition, where the term $\mathbb{E}_{z \sim q_{\phi}}\nabla_\phi\log q_\phi(z)$ vanishes under the usual regularity conditions.
\begin{equation}
\nabla_\phi \mathcal L(\phi)
=\mathbb E_{q(\epsilon)q(u)}
\Big(\nabla_z\log p(\mathcal{D},z)-\nabla_z\log q_{\phi}(z)\Big)\Big|_{z=h_{\phi}(u, \epsilon)}\cdot\nabla_{\phi}h_{\phi}(u, \epsilon),
\label{eq:grad-split}
\end{equation}

The score of the joint distribution $\log p(\mathcal{D}, z)$ is accessible. 
The main difficulty lies in the term $s_{\phi}=\nabla_{z}\log q_{\phi}(z)$, which is intractable due to the mixture integral.
UIVI resolves this by using the identity
\begin{equation}
\nabla_z\log q_\phi(z)
=\mathbb E_{\epsilon \sim q_\phi(\epsilon| z)}
\big[\nabla_z\log q_\phi(z| \epsilon)\big]
\label{eq:score-identity}
\end{equation}
expressing the marginal score as an expectation under the reverse conditional $q_\phi(\epsilon| z) \propto q(\epsilon)q_\phi(z|\epsilon)$. This setting is also applicable to sampling from an unnormalized distribution: one only needs to replace \(\nabla_z\log p(\mathcal D,z)\) with the gradient of the log unnormalized density.

In practice, UIVI estimates the expectation in~\eqref{eq:score-identity} using an inner-loop MCMC sampler targeting $q_\phi(\epsilon| z)$, which can be computationally expensive in high-dimensional settings.
\citet{titsias2019unbiased} argue that only a few MCMC steps may suffice because chains are initialized at a plausible state.
However, the samples may fail to decorrelate when $q_\phi(\epsilon| z)$ has complex geometry.
Importance-sampling variants replace MCMC samples with samples from a proposal distribution~\citep{Sobolev2019ImportanceWH}, while AISIVI~\citep{pielok2025revisiting} takes a step further by incorporating CNFs to approximate the reverse conditional.
These methods still rely on an inner approximation to $q_\phi(\epsilon|z)$, so the resulting score estimator can still suffer from numerical instability and high variance, especially when the reverse distribution has not yet been approximated well.

\section{Methodology}
In this section, we present Denoising Implicit Variational Inference (DIVI).
The standard pathwise ELBO gradient in~\eqref{eq:grad-split} is available once the marginal score $s_\phi(z)=\nabla_z\log q_\phi(z)$ is known.
The proposed component of DIVI is to learn a proxy score network $s_\psi(z)$ for this marginal score, avoiding the reverse-conditional sampling used by UIVI.
Training alternates between two steps: a $\psi$-step that fits $s_\psi(z)$ using tractable conditional scores, and a $\phi$-step that updates the semi-implicit variational distribution using a plug-in ELBO gradient.

\paragraph{$\psi$-step (learning the marginal score)} For a fixed variational parameter $\phi$, the conditional score $s_\phi(z|\epsilon)=\nabla_z\log q_\phi(z|\epsilon)$ is tractable, whereas the marginal score $s_\phi(z)$ is not.
By the score identity in~\eqref{eq:score-identity}, averaging the conditional score over the reverse conditional recovers the marginal score.
This motivates fitting $s_\psi(z)$ by regressing against the conditional score on samples generated from the semi-implicit hierarchy.
We therefore minimize the denoising score matching objective 
$$\min_{\psi} \mathbb{E}_{\epsilon \sim q(\epsilon), u \sim q(u)} \left[ \| s_{\psi}(z) - \nabla_z\log q_{\phi}(z|\epsilon)\|^2\right]$$
This objective is tractable because it only requires samples from $q(\epsilon)q(u)$ and evaluations of the conditional score.
Let the population loss for this step be denoted as $\mathcal{J}(\phi, \psi)$.

\begin{proposition}
\label{prop:pf-dsm}
Denote $\mathcal{E}^2(\phi, \psi)=\mathbb{E}_{q_{\phi}(z)}\|s_{\psi}(z)-\nabla_z\log q_{\phi}(z)\|^2$. Then we have
$$\mathcal{J}(\phi, \psi)=\mathcal{E}^2(\phi, \psi)+ \mathbb{E}_{z}\{\mathcal{I}(q_{\phi}(\epsilon|z))\}$$
where $\mathcal{I}(q_{\phi}(\epsilon|z)):=\mathbb{E}_{q_\phi(\epsilon| z)}\|\nabla_z \log q_{\phi}(\epsilon|z)\|^2$.

\end{proposition}
For fixed $\phi$, the second term in Proposition~\ref{prop:pf-dsm} is independent of $\psi$.
Thus, minimizing the tractable denoising score matching objective $\mathcal{J}(\phi,\psi)$ with respect to $\psi$ is equivalent to minimizing the marginal score error $\mathcal{E}^2(\phi,\psi)$.
In particular, the optimal score model satisfies $s_{\psi^*}(z)=\nabla_z\log q_\phi(z)$ $q_\phi$-a.e., yielding $\mathcal{E}(\phi, \psi^*)=0$.
The proof is in Appendix~\ref{app:pf-dsm}.

\paragraph{$\phi$-step (ELBO optimization with a plug-in score)} Given the learned score network $s_\psi(z)$, we approximate the exact ELBO gradient by replacing the intractable marginal score.
Using the reparameterization $z = h_{\phi}(u, \epsilon)$ with $\epsilon \sim q(\epsilon)$ and $u \sim q(u)$, this gives the averaged surrogate gradient:
$$\hat{g}(\phi; \psi) := \frac{1}{N}\sum_{i=1}^{N} \left( \nabla_z\log p(\mathcal{D}, z_i) - s_{\psi}(z_i) \right) \nabla_{\phi}h_{\phi}(u_i, \epsilon_i)$$

We then update $\phi$ by stochastic gradient ascent on the ELBO:
$$\phi_{n+1}= \phi_n + \eta_{\phi} \hat{g}(\phi_n,\psi_n ).$$
At gradient-evaluation time, this approximation replaces reverse-conditional sampling with one evaluation of the learned score network, which, by design, reduces the computational cost of each ELBO-gradient estimate compared to previous sampling-based UIVI variants.

We implement DIVI by alternating between the score-fitting step and the ELBO-gradient step.
The complete procedure is presented in Algorithm~\ref{alg:DIVI}.

\begin{algorithm}[tb]
   \caption{DIVI Algorithm}
   \label{alg:DIVI}
\begin{algorithmic}[1]
   \STATE {\bfseries Require:} Dataset $\mathcal{D}$, Generator $z = h_\phi(u, \epsilon)$, Score Network $s_\psi(z)$, Priors $q(\epsilon), q(u)$.
   \STATE {\bfseries Require:} Learning rates $\eta_\phi, \eta_\psi$, Inner steps $L_\psi$, Batch sizes $N, B$.
   \STATE Initialize $\phi_0, \psi_{-1}$.
   \FOR{$n=0,1,2,\ldots$ until convergence}
   \STATE Set $\phi \leftarrow \phi_n$ and $\psi \leftarrow \psi_{n-1}$.
   \STATE {\color{gray} // $\psi$-step: fitting the marginal score.}
   \FOR{$k=1$ {\bfseries to} $L_\psi$}
   \STATE Sample $\epsilon^{(m)} \sim q(\epsilon), u^{(m)} \sim q(u), \; z^{(m)} \leftarrow h_\phi(u^{(m)}, \epsilon^{(m)})$ for $m=1 \dots B$.
   \STATE $\mathcal{J}_\psi \leftarrow \frac{1}{B} \sum_{m=1}^B \left\| s_\psi(z^{(m)}) - \nabla_z\log q_\phi(z^{(m)}|\epsilon^{(m)}) \right\|^2$.
   \STATE $\psi \leftarrow \psi - \eta_\psi \nabla_\psi \mathcal{J}_\psi$.
   \ENDFOR
   \STATE Set $\psi_n \leftarrow \psi$.
   
   \STATE {\color{gray} // $\phi$-step: optimizing the ELBO.}
   \STATE Sample $ \epsilon^{(i)} \sim q(\epsilon), u^{(i)} \sim q(u), \; z^{(i)} \leftarrow h_\phi(u^{(i)}, \epsilon^{(i)})$ for $i=1 \dots N$.
   \STATE $\hat{g}_n \leftarrow \frac{1}{N}\sum_{i=1}^{N}\left( \nabla_z \log p(\mathcal{D}, z^{(i)}) - s_{\psi_n}(z^{(i)}) \right) \nabla_\phi h_\phi(u^{(i)}, \epsilon^{(i)})$.
   \STATE $\phi_{n+1} \leftarrow \phi_n + \eta_\phi \hat{g}_n$.
   \ENDFOR
\end{algorithmic}
\end{algorithm}

To compare the different strategies concerning gradient approximation, we list the key features of related methods in \Cref{tab:comparison}.
The original SIVI provides a surrogate lower bound~\citep{yin2018semi}, not directly optimizing the ELBO.
However, the gradient of the surrogate lower bound also gives an approximation of the ELBO gradient~\citep{pielok2025revisiting}, with $\epsilon$ sampled from the prior distribution.

\begin{table}[bt]
\begin{threeparttable}
\centering
\caption{Comparison of gradient estimation strategies.}
\label{tab:comparison}
\renewcommand{\arraystretch}{1.8} 
\footnotesize 
\begin{tabular}{lllc}
\toprule
\textbf{Method} & \textbf{Score Estimator }$\hat{s}_{\phi}(z)$& \textbf{Estimate Approach} &\textbf{Strategy}\\
\midrule
SIVI & 

$\nabla_z\log\left(\frac{1}{K}\sum_{i=1}^{K}q_{\phi}(z|\epsilon_i)\right)$

&$\epsilon_{i}\overset{\text{i.i.d.}}{\sim}q(\epsilon)$& 
Surrogate Lower Bound\\
\midrule

UIVI  & 
$\frac{1}{K}\sum_{i=1}^{K}\nabla_z\log q_{\phi}(z|\epsilon_i)$& $\epsilon_i\overset{}{\sim}q_{\phi}(\epsilon|z)\text{ (approximately)}$ & 
Hamiltonian Monte Carlo \\

\midrule
AISIVI & 
$\nabla_z\log\left(\frac{1}{K}\sum_{i=1}^{K} \frac{q(\epsilon_i)q_\phi(z|\epsilon_i)}{q_\psi(\epsilon_i|z)} \right)$& $\epsilon_i\overset{\text{i.i.d.}}{\sim}q_{\psi}(\epsilon|z)$ &

Importance Sampling \\

\midrule
\textbf{DIVI} & 
${s_{\psi}(z)}$& 
${\min_{\psi} \mathbb{E}\left[ \| s_{\psi}(z) - \nabla_z\log q_{\phi}(z|\epsilon)\|^2\right]}$
&
\textbf{Denoising Score Matching} \\
\bottomrule
\end{tabular}

\begin{tablenotes}
\footnotesize
\item[*] We compute $\nabla_{\phi}\mathcal{L}_{\scriptsize{K-1}}$ of SIVI under the reparameterization framework, which provides an approximation for $\nabla_{\phi}\mathcal{L}$. 

\end{tablenotes}

\end{threeparttable}
\end{table}


\input{proof.tex}

\input{related_works.tex}

\input{experiments.tex}

\section{Conclusion}
\label{sec:conclusion}


We introduced Denoising Implicit Variational Inference (DIVI), an ELBO-based method for semi-implicit variational families that replaces reverse-conditional sampling with direct marginal-score learning.
By fitting a score network via denoising score matching against tractable conditional scores, DIVI yields a simple plug-in pathwise ELBO-gradient estimator and avoids the inner sampling step used by previous UIVI-style methods.
We analyzed the resulting procedure as inexact stochastic gradient ascent, showing that the averaged stationarity measure is controlled by the standard optimization term and the average score-matching error.
Empirically, DIVI performs competitively across synthetic targets, conditioned diffusion posteriors, and Bayesian neural-network regression benchmarks, while achieving favorable wall-clock time in high-dimensional settings.
Overall, these results suggest that denoising score estimation provides a practical route to scalable ELBO optimization with semi-implicit variational distributions.
As a limitation, our analysis relies on smoothness of the target objective and boundedness of the optimization trajectory, and finite-capacity score networks may leave non-vanishing residual score-estimation error. We also acknowledge the performance in complex multimodal settings as a to-be-improved aspect, which might be mitigated by complex hierarchical variational family, though we leave these for future work.

\begin{ack}

This work was supported by National Natural Science Foundation of China (grant no. 12671341, grant no. 12292980 and grant no. 12292983).
The research of Cheng Zhang was support in part by National Engineering Laboratory for Big Data Analysis and Applications, the Key Laboratory of Mathematics and Its Applications (LMAM) and the Key Laboratory of Mathematical Economics and Quantitative Finance (LMEQF) of Peking University.
The authors thank Zhifei Li and Haoyu Tao for the helpful discussions and feedback in the preliminary stages of this work.
\end{ack}

\medskip
\bibliographystyle{plainnat}
\bibliography{references}

\newpage
\appendix

\input{proof_appendix}

\input{experiment_appendix}


\end{document}

%% file: proof.tex
\section{Theoretical Analysis}\label{sec:theory}





In this section, we analyze the optimization error introduced by replacing the exact marginal score with the learned score network.
The convergence analysis of black-box variational inference (BBVI) has been studied in general settings \citep{Ranganath14Operator, kim2024convergenceblackboxvariationalinference, domke2023provable}, but convergence guarantees for specific semi-implicit approaches remain limited \citep{moens2021efficientsivi, cheng2024kernel}.
For ELBO-based semi-implicit methods such as UIVI and AISIVI, the main difficulty is to control the error of the score estimator used inside the pathwise gradient.
DIVI makes this error explicit through the score-matching residual $\mathcal{E}(\phi,\psi)$ and the corresponding optimization process can be viewed as an inexact stochastic gradient procedure, allowing us to use tools for biased stochastic gradient ascent \citep{Ajalloeian20biased, Demidovich23zoo}.

We use the Gaussian reparameterized form
\(z=h_\phi(u,\epsilon)=\mu_\phi(\epsilon)+u\odot\sigma_\phi(\epsilon)\), where
\(u\sim\mathcal N(0,I_d)\) and \(\epsilon\in\mathbb R^{d_\epsilon}\) is typically Gaussian or uniform. At iteration \(n\), let \(\mathcal F_n\) denote the filtration generated by all randomness before the score estimator is updated. After obtaining \(\psi_n\), define the enlarged filtration \(\mathcal G_n=\mathcal F_n\vee\sigma(\psi_n)\). We use the following assumptions.

\begin{assumption}
\label{assump:smoothness}
The objective \(\mathcal L(\phi)\) is twice continuously differentiable and \(L\)-smooth with respect to \(\phi\).
\end{assumption}

\begin{assumption}
\label{assump:bounded_score}
There exists a constant \(M>0\) such that
$\mathbb E_{z\sim q_\phi}
\|\nabla_z\log p(\mathcal D,z)\|^4
\le M$, $
\mathbb E_{z\sim q_\phi}
\|s_\psi(z)\|^4
\le M$ holds uniformly for $\phi, \psi$. 
\end{assumption}

\begin{assumption}
\label{assump:bounded_reparam}
There exists a constant \(C>0\) such that
$
\max\left\{
\mathbb E_{\epsilon}\|\nabla_{\phi}\sigma_{\phi}(\epsilon)\|_2^4,
\mathbb E_{\epsilon}\|\nabla_{\phi}\mu_{\phi}(\epsilon)\|_2^4
\right\}
\le C
$ uniformly holds for $\phi$, 
where \(\|\cdot\|_2\) denotes the operator norm.
\end{assumption}
The uniform boundedness conditions are understood on the algorithmic domain, where the iteration sequence $(\phi_n,\psi_n)$ is assumed to reside almost surely in a deterministic compact domain. 
In practice, this can be guaranteed by projected updates, which are standard in stochastic optimization analysis \citep{ghadimi2016mini}. 
The smoothness condition is standard in non-convex stochastic optimization. 
Assumption~\ref{assump:bounded_score} controls the fourth moment of the score-error factor in the surrogate gradient through a uniform moment bound in expectation. 
We provide an additional derivation of this assumption using the linear growth conditions  in \citep{cheng2024kernel} in Appendix~\ref{app:linear-growth}.
Assumption~\ref{assump:bounded_reparam} is specific to our Gaussian reparameterization and controls the component-wise Jacobians of \(\mu_\phi\) and \(\sigma_\phi\). 
Numerical validations along the optimization path are provided in Appendix~\ref{app: numerical-experiments}.

\begin{lemma}
\label{lem:jacobian_moment}
The fourth moment of the operator norm of the reparameterization Jacobian is bounded, i.e., 
\[
\mathbb E_{u,\epsilon}\|\nabla_\phi h_\phi(u,\epsilon)\|_2^4\le C_h,
\quad
C_h:=8\left(4\log^2(2d)+8\log(2d)+9\right)C
=\mathcal O(C(\log d)^2).
\]
\end{lemma}

We next analyze the stochastic optimization dynamics. Since \(\phi_n\) and \(\psi_n\) are
\(\mathcal G_n\)-measurable,
conditioned on \(\mathcal G_n\), the samples \(z_n^{(1)},\ldots,z_n^{(N)}\overset{i.i.d.}{\sim}q_{\phi_n}\), where \(z_n^{(i)}=h_{\phi_n}(u_n^{(i)},\epsilon_n^{(i)})\), are independently generated from the current variational distribution.
Let \(\hat g_n\) be the mini-batch surrogate gradient defined in Algorithm~\ref{alg:DIVI}. We use the conditional decomposition
\[
\hat g_n
=
\nabla_\phi\mathcal L(\phi_n)+b_n+\xi_n,
\]
where \(b_n:=\mathbb E[\hat g_n\mid\mathcal G_n]-\nabla_\phi\mathcal L(\phi_n)\) is the conditional bias, and \(\xi_n:=\hat g_n-\mathbb E[\hat g_n\mid\mathcal G_n]\) is the martingale difference noise with \(\mathbb E[\xi_n\mid\mathcal G_n]=0\). 

Denote
$
\mathcal E_n^2
:=
\mathcal E^2(\phi_n,\psi_n)
=
\mathbb E_{z\sim q_{\phi_n}}\!\left[
\|s_{\phi_n}(z_n)-s_{\psi_n}(z_n)\|^2
\mid \mathcal G_n
\right]
$, which is \(\mathcal G_n\)-measurable. We can control the bias and variance as follows:

\begin{proposition}
\label{prop:error}
For constants \(J=\sqrt{C_h}\) and \(\Sigma^2=4\sqrt{MC_h}\) we have
$\|b_n\|^2\le J\mathcal E_n^2$,
 $\mathbb E[\|\xi_n\|^2\mid\mathcal G_n]\le \frac{\Sigma^2}{N}.
$
\end{proposition}

Using these bounds, we obtain the following modified ascent lemma for DIVI optimization. Compared with the standard stochastic ascent inequality, the bound contains two additional error terms: the Monte Carlo variance term, scaled by \(\Sigma^2/N\), and the score-estimation bias term, scaled by \(\mathcal E_n^2\).

\begin{lemma}
\label{lem:ascent}
If the fixed stepsize \(\eta\le 1/(4L)\), we have
\[
\mathbb E[\mathcal L(\phi_{n+1})\mid\mathcal G_n]
\ge
\mathcal L(\phi_n)
+\frac{\eta}{4}\|\nabla_\phi\mathcal L(\phi_n)\|^2
-\frac{\eta^2L\Sigma^2}{2N}
-\frac{3\eta J}{4}\mathcal E_n^2.
\]
\end{lemma}

By summing the one-step ascent inequality over \(n=0,\ldots,T-1\), we obtain an averaged stationarity bound:

\begin{theorem}
\label{thm:convergence}
Let \(\Delta_0:=\mathcal L^*-\mathcal L(\phi_0)\). Under the assumptions above, if \(\eta\le 1/(4L)\), then
\[
\frac{1}{T}\sum_{n=0}^{T-1}\mathbb E\|\nabla_\phi\mathcal L(\phi_n)\|^2
\le
\left(\frac{4\Delta_0}{\eta T}
+\frac{2\eta L\Sigma^2}{N}
\right)+3J\left(\frac{1}{T}\sum_{n=0}^{T-1}\mathbb E\mathcal E_n^2\right).
\]
In particular, if \(T\ge 32L\Delta_0 N/\Sigma^2\), then the balancing stepsize
\(\eta\asymp \sqrt{\Delta_0 N/(L\Sigma^2T)}\) is admissible, and the stochastic optimization part is
$
\mathcal O\!\left(\Sigma\sqrt{\frac{L\Delta_0}{NT}}\right).
$
\end{theorem}

See detailed proofs in Appendix \ref{app: proofs}. Theorem~\ref{thm:convergence} separates the optimization error into two parts. The first part is the usual stochastic optimization error caused by finite mini-batch sampling; when the balancing stepsize is admissible, it decays as \(\mathcal O(T^{-1/2})\) for fixed \(N\). The second part is the averaged score-estimation error, which captures the bias introduced by replacing the intractable score \(s_{\phi_n}\) with the learned estimator \(s_{\psi_n}\). 
By Proposition~\ref{prop:pf-dsm}, the denoising score-matching objective has the marginal score as its population minimizer,
hence this residual can be reduced by sufficiently optimizing the inner score-learning problem, up to the statistical and approximation errors of the score model. We provide empirical evaluations of this quantity during the training process in Section~\ref{subsec:score-estimation}.
Thus, DIVI approaches stationarity to the extent that the learned score tracks the marginal score along the optimization path.

%% file: related_works.tex
\section{Related Work}

\paragraph{ELBO-based SIVI.}
Semi-implicit variational inference (SIVI) defines expressive variational distributions by mixing an explicit conditional distribution with an implicit mixing distribution, enabling efficient sampling while leaving the marginal density and marginal score intractable~\citep{yin2018semi,moens2021efficientsivi}.
The original SIVI objective optimizes tractable surrogate ELBO bounds, whereas UIVI targets the exact ELBO by estimating the marginal score through reverse-conditional sampling over the mixing variable~\citep{titsias2019unbiased}.
Importance-weighted and AISIVI-style methods reduce the reliance on MCMC by using proposal distributions, including expressive conditional normalizing flows~\citep{Sobolev2019ImportanceWH,pielok2025revisiting}.
DIVI is closest to these ELBO-gradient methods, but replaces inner reverse-conditional score estimation with a learned marginal-score proxy trained by denoising score matching.

\paragraph{Alternative SIVI objectives.}
Other SIVI methods avoid marginal-density evaluation by changing the optimization criterion.
Fisher-divergence-based SIVI uses score matching but leads to a minimax problem~\citep{yu2023semi}, while KSIVI uses kernel Stein discrepancy to obtain a standard minimization objective~\citep{cheng2024kernel,liu2016kernelized}.
Particle semi-implicit VI uses particles to represent the mixing distribution in a special semi-implicit family construction and approximates Euclidean--Wasserstein gradient flows~\citep{lim2025particlesemiimplicitvariationalinference}.
These approaches are complementary to DIVI: they modify the objective, representation, or gradient-flow geometry, whereas DIVI keeps the ELBO-gradient framework and addresses its marginal-score bottleneck through denoising score matching~\citep{hyvarinen2005estimationofNon-Normalizedscorematching,vincent2011aconnectionbetweenscorematching}.
Convergence guarantees for ELBO-based black-box VI and related semi-implicit methods have been studied under stochastic-gradient or surrogate-objective settings~\citep{Ranganath14Operator,kim2024convergenceblackboxvariationalinference,moens2021efficientsivi,cheng2024kernel}.
Our convergence analysis treats the resulting update as inexact ELBO-gradient optimization, with residual stationarity error controlled by the denoising score-matching error.


%% file: experiments.tex
\section{Experiments}
\label{sec:experiments}

We evaluate DIVI across posterior inference problems that increase in dimension and structure. 
The experiments investigate whether DIVI can capture visible posterior geometry on two-dimensional synthetic targets, approximate a 100-dimensional conditioned diffusion posterior accurately, and provide predictive posterior distributions for Bayesian neural network (BNN) regression.  In line with prior SIVI evaluations, we incorporate annealing techniques where appropriate. 
Unless stated otherwise, tables report the mean and standard error over 10 independent runs.

Our main comparisons are against UIVI and AISIVI \citep{titsias2019unbiased,pielok2025revisiting}.  We also include SIVI \citep{yin2018semi} and KSIVI \citep{cheng2024kernel} where their objectives apply. For the instantiation of the DIVI score model $\psi$, we use a multi-layer perceptron (MLP), which is warm-started before joint training. For each experiment, all methods share the same neural-network architecture for the semi-implicit variational family.
The variational family, evaluation metrics, and per-experiment configurations are collected in Appendix~\ref{app:experiment-details}. We also provide additional arguments to support our theoretical Assumption~\ref{assump:bounded_score} and Assumption~\ref{assump:bounded_reparam} in Appendix~\ref{app: numerical-experiments}.

\subsection{Toy two-dimensional targets}
\label{subsec:toy-experiments}

We first use low-dimensional  unnormalized distribution targets to check whether the learned semi-implicit distribution captures geometric features that are easy to inspect visually.

Following previous works~\citep{yu2023semi,li2023samplingmollifiedinteractionenergy, yu2023hierarchical}, we focus on the three plotted targets in Figure~\ref{fig:toy-scatter-grid}: \texttt{x\_shaped}, \texttt{student\_uc}, and \texttt{8\_gaussians}.  These targets probe for intersecting ridges, heavy-tailed dependence, and multi-modal coverage. We train SIVI, AISIVI, UIVI, and DIVI for 10,000 iterations and KSIVI for 50,000 iterations, with learning rate $10^{-3}$ and batch size 128. For DIVI, we train the proxy score MLP for 5 steps using 2048 samples and learning rate $10^{-3}$ after each update to the main VI model. Full per-method configurations and quantitative diagnostics are given in Appendix~\ref{app:toy-details}.

Figure~\ref{fig:toy-scatter-grid} shows that DIVI preserves the main geometry across the plotted targets, with sample clouds that closely follow the target contours in these visual comparisons.
The quantitative diagnostics in Appendix~\ref{app:toy-details} show that DIVI obtains the lowest mean on four of the six KL-style and Wasserstein metrics and the second-lowest on the remaining two. We also provide a comparison of the learned score against the true target distribution score in Appendix~\ref{app:target-score-discrepancy}, which demonstrates that DIVI achieves the lowest $L^2$ error on example X-shaped target.

\begin{figure}[t]
  \centering
  \includegraphics[width=\linewidth]{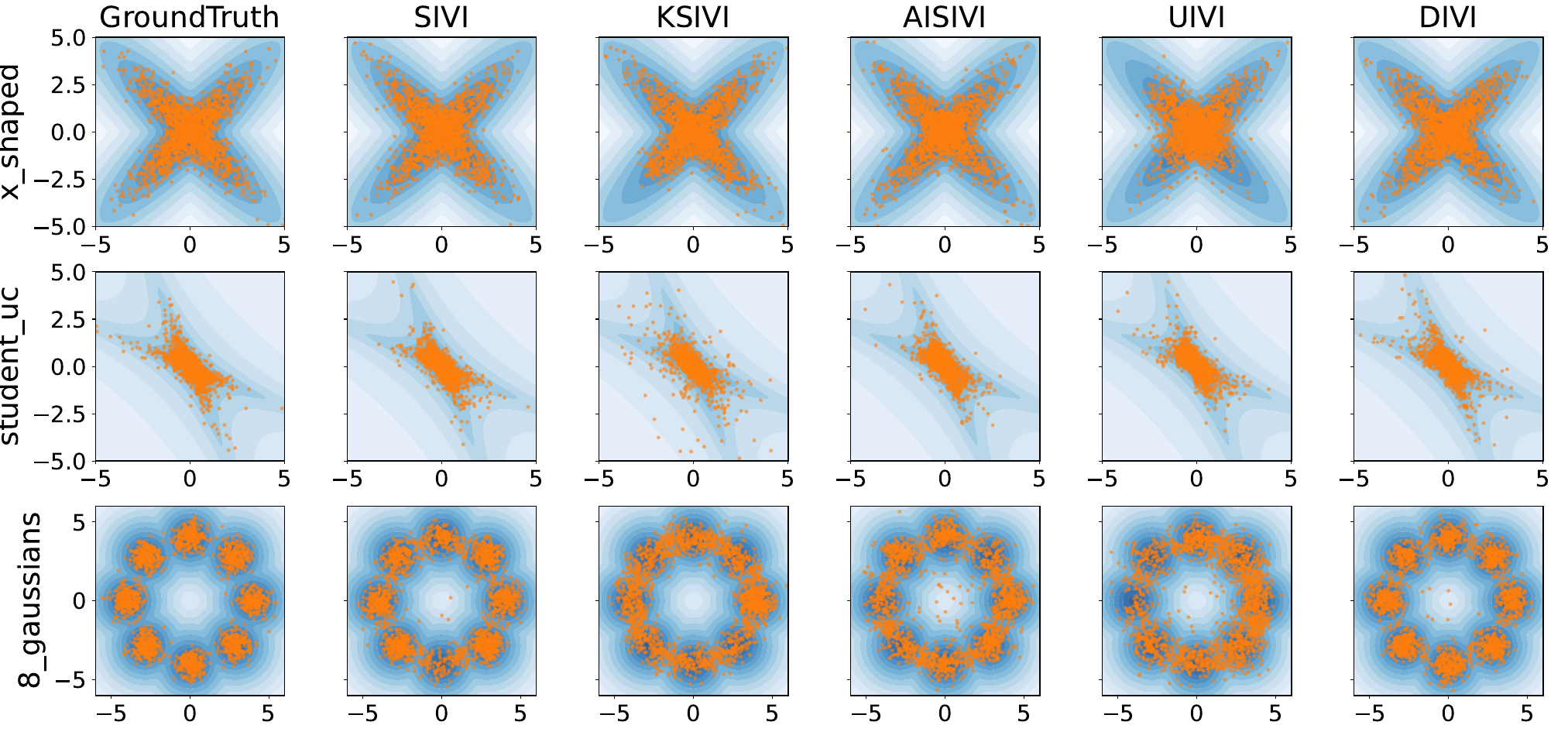}
  \caption{Qualitative comparison on two-dimensional targets.  Each panel is generated from 2,000 samples from the ground-truth baseline or a learned variational distribution; blue contours show the target density.}
  \label{fig:toy-scatter-grid}
\end{figure}

\subsection{Conditioned diffusion process}
\label{subsec:langevin-experiments}

We next evaluate a 100-dimensional conditioned diffusion posterior used in recent SIVI evaluations \citep{cheng2024kernel,pielok2025revisiting}.  The latent path follows the double-well Langevin SDE
\begin{equation}
  dx_t = 10x_t(1 - x_t^2)\,dt + dw_t,
  \qquad 0 \leq t \leq 1,
  \qquad x_0 = 0,
\end{equation}
where $w_t$ is one-dimensional Brownian motion.  We discretize the path with Euler-Maruyama step size $\Delta t=0.01$, giving
$
  x = (x_{\Delta t}, x_{2\Delta t}, \ldots, x_{100\Delta t}) \in \mathbb{R}^{100}.
$
The corresponding path prior is
\begin{equation}
  p_{\mathrm{prior}}(x)
  =
  \prod_{i=1}^{100}(2\pi\Delta t)^{-1/2}
  \exp\left\{
    -\frac{\left[x_{i\Delta t}-x_{(i-1)\Delta t}
      -10x_{(i-1)\Delta t}(1-x_{(i-1)\Delta t}^2)\Delta t\right]^2}
    {2\Delta t}
  \right\}.
\end{equation}
Observations are made every five time steps,
\begin{equation}
  y = (y_{5\Delta t}, y_{10\Delta t}, \ldots, y_{100\Delta t}),
  \qquad
  y_{5k\Delta t} \mid x_{5k\Delta t}
  \sim \mathcal{N}(x_{5k\Delta t}, \sigma^2),
\end{equation}
for $k=1,\ldots,20$ and $\sigma=0.1$.  The inference target is therefore $p(x \mid y) \propto p_{\mathrm{prior}}(x)p(y \mid x)$, with likelihood
\begin{equation}
  p(y \mid x)
  =
  \prod_{k=1}^{20}
  \frac{1}{\sqrt{2\pi\sigma^2}}
  \exp\left\{-\frac{(y_{5k\Delta t}-x_{5k\Delta t})^2}{2\sigma^2}\right\}.
\end{equation}
Posterior quality can be assessed through the mean path and the marginal uncertainty band. Following \citet{cheng2024kernel}, we approximate the posterior by running stochastic gradient Langevin dynamics (SGLD) with 1000 particles, 100,000 iterations and a step size of $10^{-4}$, and use the resulting samples as a high-quality reference for evaluation.
In the evaluation, we include SIVI and KSIVI as additional baselines. We train SIVI, UIVI, AISIVI, and DIVI for 10,000 main VI iterations with learning rate $10^{-3}$ and batch size 128; KSIVI is trained for 100,000 iterations with the same learning rate and batch size and a slower learning rate decay schedule. For DIVI, we train the proxy score MLP for 5 steps using 2048 samples and learning rate $10^{-3}$ after each update to the main VI model. Further method configurations are in Appendix~\ref{app:diffusion-details}.

\begin{figure}[t]
  \centering
  \includegraphics[width=0.9\linewidth]{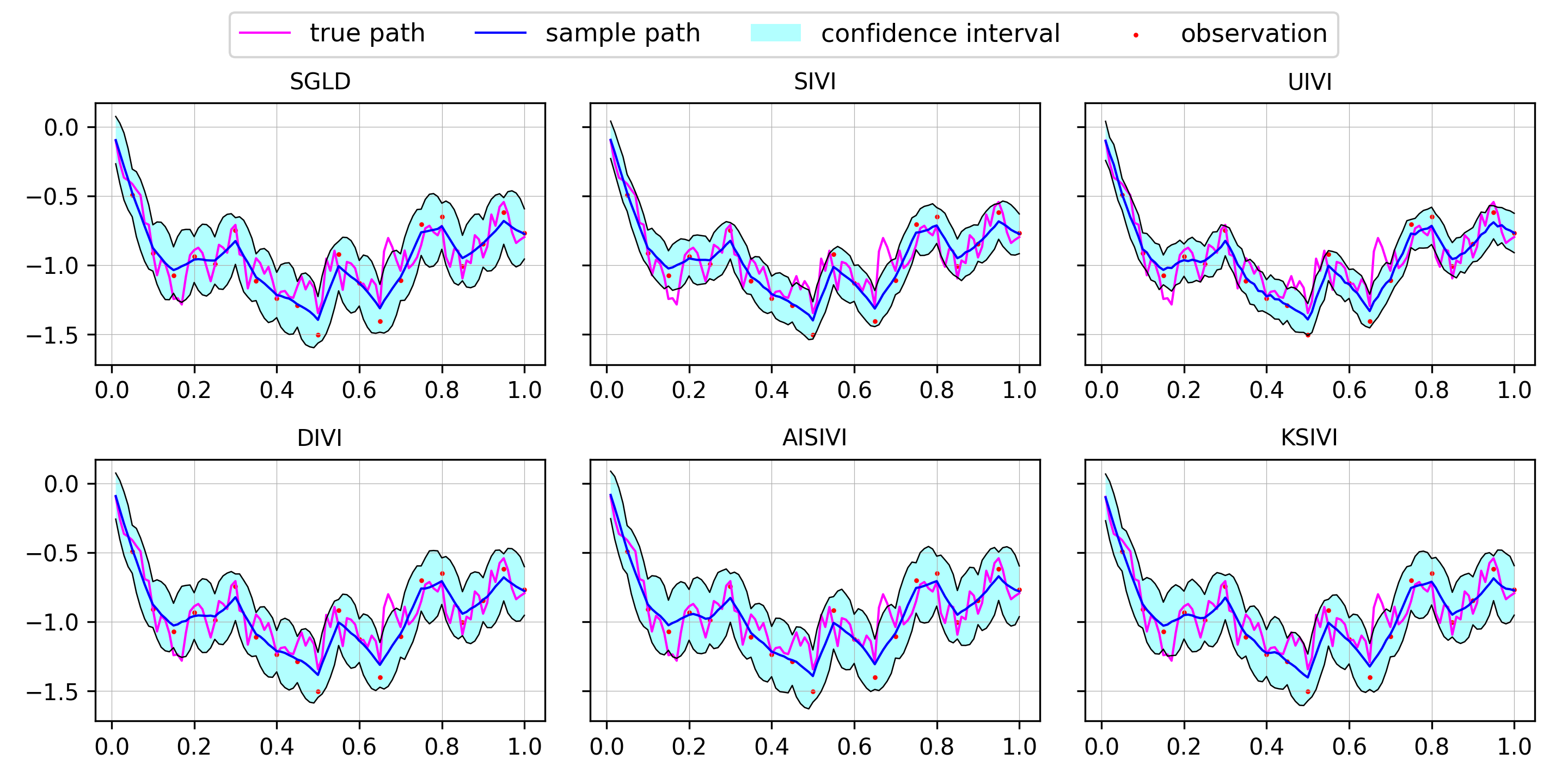}
  \vspace{-0.2em}
  \caption{Posterior path approximations on the conditioned diffusion benchmark.  The magenta curve is the true latent path, red points are observations, blue curves are posterior mean paths, and shaded regions are marginal 95\% confidence intervals.}
  \label{fig:langevin-trace-grid}
\end{figure}

\input{tables/langevin_metrics.tex}

Figure~\ref{fig:langevin-trace-grid} shows that DIVI tracks the latent trajectory while maintaining a posterior band comparable to the SGLD reference.  Following \citet{pielok2025revisiting}, Table~\ref{tab:langevin-final-eval} reports the expected log marginal likelihood of the 1,000 high-quality reference samples obtained from SGLD. For each method, we use 100,000 samples from the final trained variational model to approximate the marginal likelihood, and use the marginal likelihood estimated from 100,000 SGLD samples as the reference value. Larger values indicate better agreement with the reference posterior. We also report the wall-clock end-to-end training time for each method, including any warmup phases. As shown in Table~\ref{tab:langevin-final-eval}, DIVI obtains the highest expected log marginal likelihood among the VI methods and is closest to the SGLD reference, while also using the lowest wall-clock training time among the reported variational methods. Further metric and configuration details can be found in Appendix~\ref{app:diffusion-details}.

\subsection{Bayesian neural-network regression}
\label{subsec:bnn-experiments}

We finally evaluate posterior inference for Bayesian neural networks (BNN) on six UCI regression targets: Boston Housing, Concrete, Power Plant, Protein, Wine Red, and Yacht.  Following the BNN settings used for KSIVI and SIVI-SM \citep{cheng2024kernel,yu2023semi}, each target uses a single-hidden-layer ReLU network with 50 hidden units.  These targets provide a diverse high-dimensional suite, with posterior dimensionality ranging from 301 to 751.  Because these posteriors are data-dependent, we evaluate posterior predictive performance through test root mean square error (RMSE) and negative test log-likelihood (NLL), where lower values indicate better performance. We select the learning rate for each method on each target independently from $\{10^{-5}, 10^{-4}, 10^{-3}\}$ via cross-validation. We train all methods for 10,000 VI iterations except KSIVI, which, following the setting in \citet{cheng2024kernel}, is trained for 100,000 iterations. We train the proxy score MLP for DIVI for 2 steps using 2048 samples and learning rate $10^{-3}$ after each optimization step to the variational distribution. Dataset and evaluation details are in Appendix~\ref{app:bnn-details}.

\input{tables/bnn_rmse_nll.tex}

Table~\ref{tab:bnn-rmse-nll} gives the full predictive comparison across SIVI, UIVI, AISIVI, KSIVI, and DIVI.  DIVI is best or second-best on most reported RMSE and NLL entries, with particularly strong performance on Power and Protein.  The results also show cases where KSIVI remains stronger, such as Wine Red and Yacht RMSE.  Overall, DIVI remains competitive with the strongest baselines while requiring the lowest average wall-clock training time.

\subsection{Ablation studies}
\label{subsec:score-estimation}

First we ablate on the effect of the score estimators of each method. We evaluate SIVI, UIVI, AISIVI, and DIVI as score estimators on the same fixed variational distributions obtained after 2,000 to 10,000 DIVI updates on the X-shaped target. Table~\ref{tab:score-estimation-error} reports the squared $L^2$ discrepancy of the estimated score with respect to a reference score produced by UIVI-styled long HMC chains, together with the score-evaluation time of each method. DIVI achieves the lowest mean score-estimation error at every checkpoint and the lowest score-evaluation latency under the tested configurations. We may also observe that the score estimates of DIVI improves continually throughout the training process. More details of the evaluation are given in Appendix~\ref{app:score-estimation}. We also provide a comparison of the learned score against the true target distribution score along the trajectory of each method in Appendix~\ref{app:target-score-discrepancy}, which demonstrates that DIVI achieves the lowest $L^2$ error on all reported checkpoints on the example X-shaped target.

\begin{table}[!htbp]
\centering
\caption{Score-estimation error on fixed DIVI variational checkpoints for the X-shaped target (mean $\pm$ standard error over five seeds). Bold indicates the lowest mean.}
\label{tab:score-estimation-error}
\small
\setlength{\tabcolsep}{3pt}
\begin{adjustbox}{max width=\linewidth}
\begin{tabular}{lcccc}
\toprule
Checkpoint & SIVI & UIVI & AISIVI & DIVI \\
\midrule
2,000 steps & 154.381 {\footnotesize $\pm$ 68.264} & 1188.717 {\footnotesize $\pm$ 509.566} & 45.529 {\footnotesize $\pm$ 34.394} & \textbf{0.899} {\footnotesize $\boldsymbol{\pm}$ \textbf{0.406}} \\[2pt]
4,000 steps & 113.301 {\footnotesize $\pm$ 27.096} & 889.808 {\footnotesize $\pm$ 239.577} & 13.478 {\footnotesize $\pm$ 8.219} & \textbf{0.802} {\footnotesize $\boldsymbol{\pm}$ \textbf{0.388}} \\[2pt]
6,000 steps & 78.763 {\footnotesize $\pm$ 17.876} & 880.869 {\footnotesize $\pm$ 314.731} & 7.249 {\footnotesize $\pm$ 2.089} & \textbf{0.407} {\footnotesize $\boldsymbol{\pm}$ \textbf{0.211}} \\[2pt]
8,000 steps & 81.909 {\footnotesize $\pm$ 27.307} & 972.443 {\footnotesize $\pm$ 358.567} & 20.203 {\footnotesize $\pm$ 6.899} & \textbf{0.303} {\footnotesize $\boldsymbol{\pm}$ \textbf{0.137}} \\[2pt]
10,000 steps & 107.791 {\footnotesize $\pm$ 21.953} & 899.236 {\footnotesize $\pm$ 273.263} & 6.881 {\footnotesize $\pm$ 3.126} & \textbf{0.331} {\footnotesize $\boldsymbol{\pm}$ \textbf{0.245}} \\[2pt]
\midrule
Time (ms) & 5.316 {\footnotesize $\pm$ 0.082} & 204.523 {\footnotesize $\pm$ 10.375} & 8.096 {\footnotesize $\pm$ 0.099} & \textbf{0.134} {\footnotesize $\boldsymbol{\pm}$ \textbf{0.001}} \\
\bottomrule
\end{tabular}
\end{adjustbox}
\end{table}

Additionally, we ablate on the expressiveness of the semi-implicit variational distributional family, providing comparison with a RealNVP variational family in Appendix~\ref{app:flow-comparison}. We also perform ablations of DIVI's choice of score-update count in each variational optimization step and the reverse batch size in Appendix~\ref{app:reverse-updates}. We believe this would provide practical guidance for future users of DIVI, as the score-update count and reverse batch size are two important hyperparameters in the accuracy-computaion tradeoff of DIVI.

%% file: tables/langevin_metrics.tex
\begin{table}[t]
\centering
\caption{Conditioned diffusion results for expected log marginal likelihood and wall-clock time.}
\label{tab:langevin-final-eval}
\begin{tabular}{lccc}
\toprule
Method & ELM & wall-clock time (s) & iterations \\
\midrule
SGLD & 75.316 $\pm$ {\footnotesize 0.00} & -- & -- \\
\midrule
SIVI & 58.309 $\pm$ {\footnotesize 0.11} & 368 & 10K \\
UIVI & 21.866 $\pm$ {\footnotesize 0.23} & 756 & 10K \\
AISIVI & 74.687 $\pm$ {\footnotesize 0.03} & 736 & 10K \\
KSIVI & 74.808 $\pm$ {\footnotesize 0.04} & 948 & 100K \\
DIVI & \textbf{75.162} $\pm$ {\footnotesize 0.02} & \textbf{335} & 10K \\
\bottomrule
\end{tabular}
\end{table}

%% file: tables/bnn_rmse_nll.tex
\begin{table}[t]
\centering
\caption{BNN test performance and computational cost. RMSE and NLL are reported as mean $\pm$ standard error. Lower is better. \textbf{Bold} indicates the best result; \underline{underline} indicates the second best.}
\label{tab:bnn-rmse-nll}
\vspace{0.75em}
\small
\setlength{\tabcolsep}{4.5pt}
\renewcommand{\arraystretch}{1.08}
\begin{tabular}{llccccc}
\toprule
Dataset & Metric & SIVI & UIVI & AISIVI & KSIVI & DIVI \\
\midrule
\multirow{2}{*}{Boston} & RMSE & 2.826 $\pm$ {\footnotesize 0.06} & 3.266 $\pm$ {\footnotesize 0.11} & 3.294 $\pm$ {\footnotesize 0.10} & \textbf{2.661} $\pm$ {\footnotesize 0.02} & \underline{2.677} $\pm$ {\footnotesize 0.06} \\
 & NLL & 2.518 $\pm$ {\footnotesize 0.01} & 2.606 $\pm$ {\footnotesize 0.02} & 2.615 $\pm$ {\footnotesize 0.02} & \underline{2.490} $\pm$ {\footnotesize 0.00} & \textbf{2.477} $\pm$ {\footnotesize 0.01} \\
\addlinespace[2pt]
\multirow{2}{*}{Concrete} & RMSE & 6.932 $\pm$ {\footnotesize 0.09} & 7.668 $\pm$ {\footnotesize 0.16} & 8.243 $\pm$ {\footnotesize 0.28} & \underline{6.351} $\pm$ {\footnotesize 0.02} & \textbf{6.342} $\pm$ {\footnotesize 0.08} \\
 & NLL & 3.373 $\pm$ {\footnotesize 0.02} & 3.490 $\pm$ {\footnotesize 0.03} & 3.539 $\pm$ {\footnotesize 0.04} & \textbf{3.262} $\pm$ {\footnotesize 0.00} & \underline{3.283} $\pm$ {\footnotesize 0.02} \\
\addlinespace[2pt]
\multirow{2}{*}{Power} & RMSE & \underline{3.936} $\pm$ {\footnotesize 0.01} & 3.977 $\pm$ {\footnotesize 0.01} & 4.007 $\pm$ {\footnotesize 0.01} & 4.265 $\pm$ {\footnotesize 0.14} & \textbf{3.878} $\pm$ {\footnotesize 0.01} \\
 & NLL & \underline{2.802} $\pm$ {\footnotesize 0.00} & 2.813 $\pm$ {\footnotesize 0.00} & 2.823 $\pm$ {\footnotesize 0.00} & 2.931 $\pm$ {\footnotesize 0.06} & \textbf{2.781} $\pm$ {\footnotesize 0.00} \\
\addlinespace[2pt]
\multirow{2}{*}{Protein} & RMSE & \underline{4.892} $\pm$ {\footnotesize 0.03} & 5.021 $\pm$ {\footnotesize 0.00} & 4.934 $\pm$ {\footnotesize 0.02} & 4.967 $\pm$ {\footnotesize 0.00} & \textbf{4.493} $\pm$ {\footnotesize 0.00} \\
 & NLL & \underline{3.019} $\pm$ {\footnotesize 0.00} & 3.029 $\pm$ {\footnotesize 0.00} & 3.028 $\pm$ {\footnotesize 0.00} & 3.021 $\pm$ {\footnotesize 0.00} & \textbf{2.927} $\pm$ {\footnotesize 0.00} \\
\addlinespace[2pt]
\multirow{2}{*}{Wine Red} & RMSE & 0.623 $\pm$ {\footnotesize 0.01} & 0.627 $\pm$ {\footnotesize 0.00} & 0.625 $\pm$ {\footnotesize 0.01} & \textbf{0.604} $\pm$ {\footnotesize 0.00} & \underline{0.613} $\pm$ {\footnotesize 0.00} \\
 & NLL & 0.941 $\pm$ {\footnotesize 0.01} & 0.947 $\pm$ {\footnotesize 0.01} & 0.943 $\pm$ {\footnotesize 0.01} & \textbf{0.915} $\pm$ {\footnotesize 0.00} & \underline{0.930} $\pm$ {\footnotesize 0.01} \\
\addlinespace[2pt]
\multirow{2}{*}{Yacht} & RMSE & 1.911 $\pm$ {\footnotesize 0.12} & 2.720 $\pm$ {\footnotesize 0.10} & 2.539 $\pm$ {\footnotesize 0.10} & \textbf{1.274} $\pm$ {\footnotesize 0.07} & \underline{1.358} $\pm$ {\footnotesize 0.04} \\
 & NLL & 1.862 $\pm$ {\footnotesize 0.06} & 2.281 $\pm$ {\footnotesize 0.08} & 2.181 $\pm$ {\footnotesize 0.07} & \underline{1.595} $\pm$ {\footnotesize 0.04} & \textbf{1.557} $\pm$ {\footnotesize 0.06} \\
\midrule
\multicolumn{2}{l}{Avg. wall-clock time} & 487 & 842 & 704 & 1100 & \textbf{264} \\
\bottomrule
\end{tabular}
\end{table}

%% file: proof_appendix.tex
\section{Proofs of Propositions}
\label{app: proofs}
\subsection{Proof of Proposition 3.1}
\label{app:pf-dsm}
\begin{proof}
    Expanding the squared norm with $\nabla_z \log q_{\phi}(z|\epsilon) - \nabla_z\log q_{\phi}(z) = \nabla_z \log q_{\phi}(\epsilon|z)$:
$$
\begin{aligned}
\mathbb{E}_{z, \epsilon} \|s_{\psi}(z) - \nabla_z\log q_{\phi}(z|\epsilon)\|^2 
&= \mathbb{E}_{z, \epsilon} \|(s_{\psi}(z) - \nabla_z\log q_{\phi}(z)) - \nabla_z \log q_{\phi}(\epsilon|z)\|^2 \\
&= \underbrace{\mathbb{E}_{z} \|s_{\psi}(z) - \nabla_z\log q_{\phi}(z)\|^2}_{\mathcal{E}^2(\phi, \psi)} + \underbrace{\mathbb{E}_{z, \epsilon} \|\nabla_z \log q_{\phi}(\epsilon|z)\|^2}_{\mathbb{E}_z \mathcal{I}(q_{\phi}(\epsilon|z))} \\
&\quad - 2 \mathbb{E}_{z} \big{\langle} s_{\psi}(z) - \nabla_z \log q_{\phi}(z), \underbrace{\mathbb{E}_{q_{\phi}(\epsilon|z)} \nabla_z \log q_{\phi}(\epsilon|z)}_{\text{zero}} \big{\rangle} \\
&= \mathcal{E}^2(\phi, \psi) + \mathbb{E}_{z} \mathcal{I}(q_{\phi}(\epsilon|z))
\end{aligned}
$$

where the last equality holds under the standard regularity conditions
that allow differentiation under the integral sign. If the score model class is sufficiently rich, the population minimizer
satisfies \(s_{\psi^*}(z)=\nabla_z\log q_\phi(z)\) \(q_\phi\)-a.e.,
and hence \(\mathcal E(\phi,\psi^*)=0\).
\end{proof}

\subsection{Proof of Lemma 4.1}
\label{app:pf-err}
\begin{proof}

Under the reparameterization $z=\mu_{\phi}(\epsilon)+u\odot\sigma_{\phi}(\epsilon)$, $u\sim \mathcal{N}(0, I_{d}), \epsilon\in\mathbb{R}^{d_{\epsilon}},\phi\in\mathbb{R}^{d_{\phi}}$, the gradient $\nabla_{\phi}\mu_{\phi}(\epsilon)\in \mathbb{R}^{d\times d_{\phi}}, \nabla_{\phi}\sigma_{\phi}(\epsilon)\in\mathbb{R}^{d\times d_{\phi}}$. Throughout the proof, all vectors are treated as row vectors. 

First we give an upper bound for $\mathbb{E}\|u\|_{\infty}^4=\mathbb{E}\max_{1\leq i\leq d} |u_i|^4$.  For a standard Gaussian random variable, we use the tail bound
$
\mathbb P(|u_i|>y)\le 2e^{-y^2/2}, y>0.
$
Therefore, for any \(x>0\),
$
\mathbb P(u_i^4>x)
=
\mathbb P(|u_i|>x^{1/4})
\le
2e^{-\sqrt{x}/2}.
$ Hence,
$$
\mathbb{E}\|u\|_{\infty}^4
= \int_{0}^{\infty} \mathbb{P}\left(\max_{1 \le i \le d} u_i^4 > x\right) \mathrm{d}x 
\le \int_{0}^{\infty} \min\{1, \sum_{i=1}^{d} \mathbb{P}(u_i^4 > x)\} \mathrm{d}x \le \int_{0}^{\infty} \min\{1, 2d\cdot e^{-\sqrt{x}/2}\} \mathrm{d}x.
$$
Solving for $x_0$ such that $2d\cdot e^{-\sqrt{x}/2} = 1$ yields $x_0 = 4\log^2(2d)$. Splitting the integral at $x_0$, we obtain:
$$
\begin{aligned}
\int_{0}^{\infty} \min\{1, 2d\cdot e^{-\sqrt{x}/2}\} \mathrm{d}x
&= 4\log^2(2d) + 2d \int_{4\log^2(2d)}^{\infty} e^{-\sqrt{x}/2} \mathrm{d}x.\\
&=4\log^2(2d)+16d\int_{\log(2d)}^{\infty}t e^{-t}\mathrm{d}t\\
& =4\log^2(2d)+8\log(2d)+8.
\end{aligned}
$$

 So we arrive at the upper bound for the fourth moment $
    \mathbb{E}\|u\|_{\infty}^4 \le 4\log^2(2d) + 8\log(2d) + 8$. Based on this result, we can control the fourth order moment $\mathbb{E}\|\nabla_{\phi}h_{\phi}(u, \epsilon)\|_2^4$: 
$$\begin{aligned}
\mathbb{E}\|\nabla_{\phi}h_{\phi}(u, \epsilon)\|_2^4
		&= \mathbb{E}\|\nabla_{\phi}\mu_{\phi}(\epsilon) + \text{diag}(u)\nabla_{\phi}\sigma_{\phi}(\epsilon)\|_2^4\\ 
		&\leq 8\ \mathbb{E}  \left(\|\nabla_{\phi}\mu_{\phi}(\epsilon)\|_2^4 + \|u\|_{\infty}^4\cdot\|\nabla_{\phi}\sigma_{\phi}(\epsilon)\|_2^4 \right)\\
		&= 8\mathbb{E}_{\epsilon}\|\nabla_{\phi}\mu_{\phi}(\epsilon)\|_2^4 + 8\mathbb{E}_u\|u\|_{\infty}^4\cdot\mathbb{E}_{\epsilon}\|\nabla_{\phi}\sigma_{\phi}(\epsilon)\|_2^4   \\
		&\leq 8\mathbb{E}_{\epsilon}\|\nabla_{\phi}\mu_{\phi}(\epsilon)\|_2^4 + 8(4\log^2(2d) + 8\log(2d) + 8)\mathbb{E}_{\epsilon}\|\nabla_{\phi}\sigma_{\phi}(\epsilon)\|_2^4\\
		&\leq 8(4\log^2(2d) + 8\log(2d) + 9)C=\mathcal{O}(C(\log d)^2).
\end{aligned}$$

Here, $\|\cdot\|_2$ denotes the $2-$ operator norm for matrices, $\text{diag}(u)$ is the diagonal matrix with $\left(\text{diag}(u)\right)_{ii}=u_i$.
\end{proof}

\subsection{Bound Based on Linear Growth Condition}
\label{app:linear-growth}

We show that Assumption \ref{assump:bounded_score} can be derived from a
linear growth condition together with a bounded fourth moment condition on
the reparameterization functions.

Similar to the conditions of linear growth and fourth-moment boundedness in \citep{cheng2024kernel}, we assume that there exists a constant
\(\lambda>0\) such that
$$
\max\left\{
\|\nabla_z\log p(\mathcal D,z)\|,
\|s_\psi(z)\|
\right\}
\le
\lambda(\|z\|+1)
$$
holds uniformly for \(z\) and \(\psi\) and there exists a constant \(\zeta>0\) such that
$$
\max\left\{
\mathbb E_{\epsilon}\|\mu_\phi(\epsilon)\|^4,
\mathbb E_{\epsilon}\|\sigma_\phi(\epsilon)\|^4
\right\}
\le \zeta
$$
holds uniformly for \(\phi\).
Then Assumption \ref{assump:bounded_score} holds with
$$
M
=
8\lambda^4
\left[
8\zeta
\left(
4\log^2(2d)+8\log(2d)+9
\right)
+1\right]=\mathcal{O}(\zeta\lambda^4\log^2 d)
.
$$
\begin{proof}
It suffices to prove the bound for \(s_\psi(z)\); the bound for
\(\nabla_z\log p(\mathcal D,z)\) follows identically. By the linear growth
condition and \((a+b)^4\le 8(a^4+b^4)\),
\[
\begin{aligned}
\mathbb E_{q_\phi}\|s_\psi(z)\|^4
&\le
\lambda^4 \mathbb E_{q_\phi}(\|z\|+1)^4 \le
8\lambda^4
\left(
\mathbb E_{q_\phi}\|z\|^4+1
\right) \\
&=
8\lambda^4
\left(
\mathbb E_{u,\epsilon}
\|\mu_\phi(\epsilon)+u\odot\sigma_\phi(\epsilon)\|^4
+1
\right) \\
&\le
8\lambda^4
\left(
8\mathbb E_{\epsilon}\|\mu_\phi(\epsilon)\|^4
+
8\mathbb E_{u,\epsilon}\|u\odot\sigma_\phi(\epsilon)\|^4
+1
\right) \\
&\le
8\lambda^4
\left(
8\zeta
+
8\mathbb E_u\|u\|_\infty^4
\mathbb E_\epsilon\|\sigma_\phi(\epsilon)\|^4
+1
\right) \\
&\le
8\lambda^4
\left[
8\zeta
+
8\zeta
\left(
4\log^2(2d)+8\log(2d)+8
\right)
+1
\right] \\
&=
8\lambda^4
\left[
8\zeta
\left(
4\log^2(2d)+8\log(2d)+9
\right)
+1
\right].
\end{aligned}
\]
Therefore, $
\mathbb E_{q_\phi}\|s_\psi(z)\|^4\le M,
$
where
$
M
:=
8\lambda^4
\left[
8\zeta
\left(
4\log^2(2d)+8\log(2d)+9
\right)
+1
\right].
$
The same argument gives
$
\mathbb E_{q_\phi}
\|\nabla_z\log p(\mathcal D,z)\|^4
\le M.
$ We also provide numerical validations on the reasonable bound of $\zeta$ and $\lambda$, which is discussed concisely in Appendix \ref{app: numerical-experiments} 
\end{proof}
\subsection{Proof of Proposition 4.2}\label{app:error}
\begin{proof}
Based on the preparations in Appendix \ref{app:pf-err}, we can respectively bound the bias and variance of the surrogate gradient.

First, we fix \((\phi,\psi)\) and take all expectations with respect to
the fresh sampling variables \((u,\epsilon)\), with \(z=h_\phi(u,\epsilon)\).
Equivalently, expectations of functions depending only on \(z\) are taken
with respect to \(z\sim q_\phi\).
We concentrate on two estimators only using a single sample $z$:
\begin{equation}
    \label{unbiased}
    {g}(z;\phi)=\left(\nabla_z\log p(\mathcal{D}, z)-s_{\phi}(z)\right)\nabla_{\phi}h_{\phi}(u, \epsilon)
\end{equation}
\begin{equation}\label{biased}
    \hat{g}(z;\phi, \psi)=\left(\nabla_z\log p(\mathcal{D}, z)-s_{\psi}(z)\right)\nabla_{\phi}h_{\phi}(u, \epsilon)
\end{equation}
Since
$
\mathbb E_{u,\epsilon} g(z;\phi)
=
\nabla_\phi \mathcal L(\phi)
$, \eqref{unbiased} gives an unbiased estimator of
\(\nabla_\phi \mathcal L(\phi)\).
However, the semi-implicit score function
\(s_\phi(z)=\nabla_z\log q_\phi(z)\)
is intractable. Therefore, \eqref{biased} gives a biased estimator of $\nabla_{\phi}\mathcal{L}(\phi)$ with bias-variance decomposition:
\begin{equation}\label{equ:decomposition}
\hat{g}(z;\phi ,\psi)=\nabla_{\phi}\mathcal{L}(\phi)+b(\phi, \psi)+\xi(z;\phi, \psi)
\end{equation}
where $b(\phi, \psi)=\mathbb{E}\hat{g}(z;\phi ,\psi)-\nabla_{\phi}\mathcal{L}(\phi)$, $\xi(z;\phi, \psi)=\hat{g}(z;\phi ,\psi)-\nabla_{\phi}\mathcal{L}(\phi)-b(\phi, \psi)$. Then $\mathbb{E}\xi(z;\phi, \psi)=0.$ 

By Lemma \ref{lem:jacobian_moment} and Cauchy-Schwarz inequality, we can control the bias norm $\|b(\phi, \psi)\|$ by the score matching error $\mathcal{E}(\phi, \psi)$:
\begin{align*}
\|b(\phi, \psi)\|^2
&=
\left\|
\mathbb{E}_{u,\epsilon}
\left[
\hat{g}(z;\phi, \psi)-g(z;\phi)
\right]
\right\|^2 \\
&\leq
\left(
\mathbb{E}_{u,\epsilon}
\left\|
\hat{g}(z;\phi, \psi)-g(z;\phi)
\right\|
\right)^2 \\
&=
\left(
\mathbb{E}_{u,\epsilon}
\left\|
\left(s_{\phi}(z)-s_{\psi}(z)\right)
\nabla_{\phi}h_{\phi}(u,\epsilon)
\right\|
\right)^2,\\
&\leq
\left(
\mathbb{E}_{u,\epsilon}
\left[
\|s_{\phi}(z)-s_{\psi}(z)\|
\cdot
\|\nabla_{\phi}h_{\phi}(u,\epsilon)\|_2
\right]
\right)^2 \\
&\leq
\mathbb{E}_{z\sim q_\phi}
\|s_{\phi}(z)-s_{\psi}(z)\|^2
\cdot
\mathbb{E}_{u,\epsilon}
\|\nabla_{\phi}h_{\phi}(u,\epsilon)\|_2^2 \\
&\leq
\mathbb{E}_{z\sim q_\phi}
\|s_{\phi}(z)-s_{\psi}(z)\|^2
\cdot
\sqrt{
\mathbb{E}_{u,\epsilon}
\|\nabla_{\phi}h_{\phi}(u,\epsilon)\|_2^4
} \\
&\leq
\sqrt{C_h
}
\cdot
\mathcal{E}^2(\phi,\psi) =
J\cdot\mathcal{E}^2(\phi,\psi).
\end{align*}

Next, $\mathbb{E}\|\xi\|^2$ could be controlled similarly. 
\begin{align*}
\mathbb{E}_{u,\epsilon}\|\xi\|^2
&=
\mathbb{E}_{u,\epsilon}
\left\|
\hat{g}(z;\phi,\psi)
-
\mathbb{E}_{u,\epsilon}\hat{g}(z;\phi,\psi)
\right\|^2 \\
&\leq
\mathbb{E}_{u,\epsilon}
\|\hat{g}(z;\phi,\psi)\|^2 \\
&=
\mathbb{E}_{u,\epsilon}
\left\|
\left(\nabla_z\log p(\mathcal{D}, z)-s_{\psi}(z)\right)
\nabla_{\phi}h_{\phi}(u,\epsilon)
\right\|^2 \\
&\leq
\mathbb{E}_{u,\epsilon}
\left[
\|\nabla_z\log p(\mathcal{D}, z)-s_{\psi}(z)\|^2
\cdot
\|\nabla_{\phi}h_{\phi}(u,\epsilon)\|_2^2
\right] \\
&\leq
\sqrt{
\mathbb{E}_{z\sim q_\phi}
\|\nabla_z\log p(\mathcal{D}, z)-s_{\psi}(z)\|^4
}
\cdot
\sqrt{
\mathbb{E}_{u,\epsilon}
\|\nabla_{\phi}h_{\phi}(u,\epsilon)\|_2^4
} \\
&\leq \sqrt{8
(\mathbb{E}_{z\sim q_\phi}
\|\nabla_z\log p(\mathcal{D}, z)\|^4+\mathbb{E}_{z\sim q_{\phi}}\|s_{\psi}(z)\|^4)
}\cdot
\sqrt{
\mathbb{E}_{u,\epsilon}
\|\nabla_{\phi}h_{\phi}(u,\epsilon)\|_2^4
}\\
&\leq
4\sqrt{C_h\cdot M}=\Sigma^2.
\end{align*}

We now relate these conditional quantities to the preceding single-sample
bounds. The bounds derived above were fixed-parameter bounds: for any fixed
\((\phi,\psi)\), with fresh samples \(z=h_\phi(u,\epsilon)\), we have
$
\|b(\phi,\psi)\|^2
\le
J\mathcal E^2(\phi,\psi),
\quad
\mathbb E\|\xi(z;\phi,\psi)\|^2
\le
\Sigma^2.
$
Conditioning on \(\mathcal G_n\), the iterates \(\phi_n\) and \(\psi_n\) are
fixed, and the only remaining randomness comes from the fresh mini-batch.
Therefore, the preceding fixed-parameter bias bound can be applied with
\((\phi,\psi)=(\phi_n,\psi_n)\), yielding
$
\|b_n\|^2
\le
J\mathcal E_n^2.
$

It remains to justify the mini-batch variance bound. Define the single-sample
surrogate gradient at iteration \(n\) by
$
g_n^{(i)}
=
\left(
\nabla_z\log p(\mathcal{D},z_n^{(i)})
-
s_{\psi_n}(z_n^{(i)})
\right)
\nabla_{\phi}h_{\phi_n}(u_n^{(i)},\epsilon_n^{(i)}),
$
Then $
\hat g_n
=
\frac{1}{N}\sum_{i=1}^N g_n^{(i)},
$
and $
\xi_n
=
\frac{1}{N}\sum_{i=1}^N
\left(
g_n^{(i)}
-
\mathbb E[g_n^{(i)}\mid \mathcal G_n]
\right).
$
Conditional on \(\mathcal G_n\), the variables
\(g_n^{(1)},\ldots,g_n^{(N)}\) are i.i.d. Hence the cross terms vanish, and
\begin{align*}
\mathbb E[\|\xi_n\|^2\mid \mathcal G_n]
&=
\frac{1}{N}
\mathbb E
\left[
\left\|
g_n^{(1)}
-
\mathbb E[g_n^{(1)}\mid \mathcal G_n]
\right\|^2
\mid
\mathcal G_n
\right] 
\le
\frac{1}{N}
\mathbb E
\left[
\|g_n^{(1)}\|^2
\mid
\mathcal G_n
\right]
\le
\frac{\Sigma^2}{N}.
\end{align*}

This completes the proof.
\end{proof}

\subsection{Proof of Lemma 4.3}
\label{app:pf-ascent}

\begin{proof}
Recall that the update rule is
$\phi_{n+1}=\phi_n+\eta \hat g_n.$
By the \(L\)-smoothness of \(\mathcal L(\phi)\), we have
\begin{align*}
\mathcal L(\phi_{n+1})
&\geq
\mathcal L(\phi_n)
+
\left\langle
\phi_{n+1}-\phi_n,
\nabla_\phi \mathcal L(\phi_n)
\right\rangle
-
\frac{L}{2}
\|\phi_{n+1}-\phi_n\|^2  \\
&=
\mathcal L(\phi_n)
+
\eta
\left\langle
\nabla_\phi \mathcal L(\phi_n),
\hat g_n
\right\rangle
-
\frac{\eta^2L}{2}
\|\hat g_n\|^2 .
\end{align*}
By Proposition \ref{prop:error}, we have
 the decomposition $
\hat g_n
=
\nabla_\phi\mathcal L(\phi_n)+b_n+\xi_n,
\quad
\mathbb E[\xi_n\mid\mathcal G_n]=0,
$ 
with
$
\|b_n\|^2\leq J\mathcal E_n^2,\quad
\mathbb E[\|\xi_n\|^2\mid\mathcal G_n]\leq \frac{\Sigma^2}{N}.
$
Taking conditional expectation with respect to \(\mathcal G_n\), we obtain
\begin{align*}
\mathbb E[
\mathcal L(\phi_{n+1})
\mid
\mathcal G_n
]
&\geq
\mathcal L(\phi_n)
+
\eta
\left\langle
\nabla_\phi\mathcal L(\phi_n),
\nabla_\phi\mathcal L(\phi_n)+b_n
\right\rangle 
-
\frac{\eta^2L}{2}
\mathbb E
\left[
\left\|
\nabla_\phi\mathcal L(\phi_n)+b_n+\xi_n
\right\|^2
\mid
\mathcal G_n
\right] \\
&=
\mathcal L(\phi_n)
+
\eta
\|\nabla_\phi\mathcal L(\phi_n)\|^2
+
\eta
\left\langle
\nabla_\phi\mathcal L(\phi_n),b_n
\right\rangle 
-
\frac{\eta^2L}{2}
\left\|
\nabla_\phi\mathcal L(\phi_n)+b_n
\right\|^2
\\ &\quad- 
\frac{\eta^2L}{2}
\mathbb E[\|\xi_n\|^2\mid\mathcal G_n] \\
&\geq
\mathcal L(\phi_n)
+
\eta
\|\nabla_\phi\mathcal L(\phi_n)\|^2
-
\frac{\eta}{2}
\left(
\|\nabla_\phi\mathcal L(\phi_n)\|^2+\|b_n\|^2
\right) 
\\
&\quad
-\eta^2L
\left(
\|\nabla_\phi\mathcal L(\phi_n)\|^2+\|b_n\|^2
\right)
-
\frac{\eta^2L}{2}
\mathbb E[\|\xi_n\|^2\mid\mathcal G_n] \\
&=
\mathcal L(\phi_n)
+
\frac{\eta}{2}(1-2\eta L)
\|\nabla_\phi\mathcal L(\phi_n)\|^2
-\frac{\eta}{2}(1+2\eta L)
\|b_n\|^2-\frac{\eta^2L}{2}
\mathbb E[\|\xi_n\|^2\mid\mathcal G_n] \\
&\geq
\mathcal L(\phi_n)
+
\frac{\eta}{2}(1-2\eta L)
\|\nabla_\phi\mathcal L(\phi_n)\|^2
-
\frac{\eta J}{2}(1+2\eta L)\mathcal E_n^2
-
\frac{\eta^2L\Sigma^2}{2N}.
\end{align*}
In the last inequality, we used
$
\langle a,b\rangle
\geq
-\frac{1}{2}\|a\|^2-\frac{1}{2}\|b\|^2$,
$\|a+b\|^2\leq 2\|a\|^2+2\|b\|^2.
$

Under the stepsize condition \(\eta\leq 1/(4L)\), we have $
\frac{\eta}{2}(1-2\eta L)\geq \frac{\eta}{4},
\quad 
\frac{\eta J}{2}(1+2\eta L)\leq \frac{3\eta J}{4}.
$ Therefore,
\[
\mathbb E[
\mathcal L(\phi_{n+1})
\mid
\mathcal G_n
]
\geq
\mathcal L(\phi_n)
+
\frac{\eta}{4}
\|\nabla_\phi\mathcal L(\phi_n)\|^2
-
\frac{\eta^2L\Sigma^2}{2N}
-
\frac{3\eta J}{4}
\mathcal E_n^2.
\]
This completes the proof.
\end{proof}

\subsection{Proof of Theorem 4.4}
\label{app:pf-convergence}

\begin{proof}
Taking total expectation on both sides of Lemma \ref{lem:ascent}, we obtain
\[
\mathbb{E}\mathcal{L}(\phi_{n+1})
\geq
\mathbb{E}\mathcal{L}(\phi_n)
+
\frac{\eta}{4}
\mathbb{E}\|\nabla_{\phi}\mathcal{L}(\phi_n)\|^2
-
\frac{\eta^2L\Sigma^2}{2N}
-
\frac{3\eta J}{4}
\mathbb{E}\mathcal E_n^2.
\]
Summing from \(n=0\) to \(T-1\), we get
\[
\mathbb{E}\mathcal{L}(\phi_T)
\geq
\mathcal{L}(\phi_0)
+
\frac{\eta}{4}
\sum_{n=0}^{T-1}
\mathbb{E}\|\nabla_{\phi}\mathcal{L}(\phi_n)\|^2
-
\frac{\eta^2L\Sigma^2 T}{2N}
-
\frac{3\eta J}{4}
\sum_{n=0}^{T-1}
\mathbb E\mathcal E_n^2.
\]
Rearranging and using \(\mathbb{E}\mathcal{L}(\phi_T)\leq \mathcal{L}^*\),
\[
\frac{\eta}{4}
\sum_{n=0}^{T-1}
\mathbb{E}\|\nabla_{\phi}\mathcal{L}(\phi_n)\|^2
\leq
\Delta_0
+
\frac{\eta^2L\Sigma^2 T}{2N}
+
\frac{3\eta J}{4}
\sum_{n=0}^{T-1}
\mathbb E\mathcal E_n^2.
\]
Dividing both sides by \(\eta T/4\), we obtain
\begin{align}
\frac{1}{T}
\sum_{n=0}^{T-1}
\mathbb{E}
\|\nabla_{\phi}\mathcal{L}(\phi_n)\|^2
&\leq
\frac{4(\Delta_0)}{\eta T}
+
\frac{2\eta L\Sigma^2}{N}
+
3J
\left(
\frac{1}{T}
\sum_{n=0}^{T-1}
\mathbb E\mathcal E_n^2
\right) \\
&=
\mathcal{O}
\left(
\frac{1}{\eta T}
+
\frac{\eta}{N}
+
\frac{1}{T}\sum_{n=0}^{T-1}\mathbb{E}\mathcal{E}_n^2
\right).
\end{align}

The right-hand side consists of three terms. The first term is the finite-iteration optimization error, the second term comes from the Monte Carlo variance of the mini-batch estimator, and the third term reflects the accumulated bias caused by the imperfect score estimator \(s_{\psi_n}\). Balancing the first two terms suggests the choice
$\eta
\asymp
\min\left\{
\frac{1}{4L},
\sqrt{
\frac{N(\Delta_0)}{L\Sigma^2T}
}
\right\}$.
With the balancing stepsize, when it is admissible, the stochastic optimization part satisfies
$
\frac{4\Delta_0}{\eta T}+\frac{2\eta L\Sigma^2}{N}
=
\mathcal O\left(\Sigma\sqrt{\frac{L\Delta_0}{NT}}\right).
$ Therefore, if \(N\) is fixed and \(L,\Delta_0,\Sigma\) are treated as constants, the stochastic optimization error is of order \(\mathcal O(T^{-1/2})\).

\end{proof}

%% file: experiment_appendix.tex
\section{Experimental Details}
\label{app:experiment-details}

\subsection{Shared setup}
\label{app:shared-setup}

\paragraph{Variational family.}
All SIVI variants use a conditional Gaussian semi-implicit variational family.  We write the mixing variable as $\epsilon$ and the latent variable as $z$, with
\begin{equation}
  q_\phi(z)
  =
  \int q_\phi(z \mid \epsilon) q(\epsilon)\,d\epsilon,
  \qquad
  q_\phi(z \mid \epsilon)
  =
  \mathcal{N}\!\left(z;\mu_\phi(\epsilon),\Sigma_\phi(\epsilon)\right).
\end{equation}
For two-dimensional targets, the mixing prior is standard Gaussian, and $\mu_\phi$ and the diagonal variance are produced by a two-hidden-layer MLP with hidden width 128. For the conditioned diffusion task, $\epsilon \in [0,1]^{100}$ uses a uniform prior, $\mu_\phi$ is a two-hidden-layer ReLU MLP with hidden width 128, and the conditional covariance is a learned global diagonal variance initialized near $0.01725$.  For BNN regression, $\epsilon \in \mathbb{R}^{3}$ uses a standard Gaussian prior, $\mu_\phi$ and the diagonal log variance are produced by a two-hidden-layer ReLU MLP with hidden width 10, and the log variance is initialized at $-8$ with a lower clamp at $-20$.  These BNN choices follow the compact semi-implicit architecture used in the KSIVI and SIVI-SM comparisons.

\paragraph{Common training protocol.}
Unless otherwise noted, quantitative results use ten random seeds and report the mean and standard error. The supplementary score-error comparisons, reverse-update ablations, and RealNVP comparison use five seeds and report standard errors. For $n$ independent runs, the standard error is $s/\sqrt{n}$, where $s$ is the sample standard deviation across runs. Unless otherwise noted, SIVI, UIVI, AISIVI, and DIVI use 10,000 main VI iterations with batch size 128, Adam optimization, and a step learning-rate scheduler. Unless otherwise noted, all methods are trained with annealing enabled for half of the total training iterations. All experiments are conducted on a Linux server with Nvidia RTX 4090 GPUs and Intel Xeon Gold 6430 @ 2.10GHz.

\paragraph{SIVI.}
SIVI approximates the semi-implicit marginal $\log q_\phi(z)$ using Monte Carlo averages over auxiliary $\epsilon$ samples.  We draw 4,096 auxiliary samples on diffusion targets and 2,048 on BNN targets due to GPU memory constraints.

\paragraph{UIVI.}
UIVI samples from the reverse conditional $q_\phi(\epsilon \mid z)$ via Hamiltonian Monte Carlo.  After 5 burn-in transitions, a single HMC trajectory of 5 leapfrog steps provides $K=5$ auxiliary samples, one per intermediate position, following \citet{titsias2019unbiased}.  The leapfrog step size is 0.2.

\subsection{Metrics}
\label{app:metric-details}

\paragraph{KL-style ELBO.}
For toy and diffusion targets, we report a KL-style objective,
\begin{equation}
  D_{\mathrm{KL}}
  =
  -\widehat{\mathcal{L}}(q_\phi)
  =
  -\frac{1}{N}
  \sum_{i=1}^{N}
  \left[
    \log p(z_i) - \widehat{\log q_\phi(z_i)}
  \right],
  \qquad
  z_i \sim q_\phi .
\end{equation}
The estimate uses $N=5{,}000$ VI samples.  For each $z_i$, $\widehat{\log q_\phi(z_i)}$ is approximated with 20 batches of 2,048 auxiliary $\epsilon$ samples, plus the generating $\epsilon$ sample. We utilize the batching logic in \cite{pielok2025revisiting} to increase the effective number of Monte Carlo auxiliary samples to 40960.

\paragraph{Sliced Wasserstein-2.}
Sliced Wasserstein-2 is computed by averaging one-dimensional Wasserstein distances over random unit projections,
\begin{equation}
  \mathrm{SW}_2^2(q,r)
  =
  \mathbb{E}_{\theta \sim \mathrm{Unif}(\mathbb{S}^{d-1})}
  \left[
    W_2^2\!\left(\theta^\top q,\theta^\top r\right)
  \right],
\end{equation}
estimated using 10,000 samples and 5,000 random projections.  For heavy-tailed distributions, we report a truncated variant following \cite{li2023samplingmollifiedinteractionenergy}: both variational and reference samples are retained only if every coordinate has absolute value less than a threshold, the \textit{edge length}, and sliced Wasserstein-2 is then computed on the accepted samples.  This provides a stable central-region transport diagnostic when tail draws dominate finite-sample Wasserstein estimates.  We use 6 as the edge length for \texttt{8\_gaussians} and edge length 8 for \texttt{student\_uc}. 

\paragraph{Expected log marginal (ELM).}
For the conditioned diffusion benchmark, we report the expected log marginal likelihood,
\begin{equation}
  \mathrm{ELM}
  =
  \mathbb{E}_{z \sim r}
  \left[
    \sum_{j=1}^{d}
    \log \widehat{q}_{\phi,j}(z_j)
  \right],
\end{equation}
where $r$ is the empirical distribution from the SGLD reference samples and $\widehat{q}_{\phi,j}$ is a one-dimensional Gaussian KDE fit to coordinate $j$ of model samples.  We follow \citet{pielok2025revisiting} in using this metric.  Direct Monte Carlo approximation of the full semi-implicit marginal likelihood is difficult because it integrates over a 100-dimensional mixing distribution, so the coordinate-wise KDE estimate provides a stable diagnostic, though it does not capture the full joint marginal likelihood. The reported evaluation uses 1,000 reference samples and 100,000 model samples; the SGLD baseline uses 100,000 SGLD samples as the model sample set.

\paragraph{BNN predictive metrics.}
For BNN targets, evaluation draws 2,000 posterior samples (network weights) from the learned variational distribution.  RMSE is computed from the posterior predictive mean on the test set.  Test log-likelihood is estimated by Monte Carlo marginalization over posterior samples, and the reported NLL is its negative.

\subsection{Two-dimensional targets}
\label{app:toy-details}

The main text displays and summarizes \texttt{x\_shaped}, \texttt{student\_uc}, and \texttt{8\_gaussians}. Visual panels use 2,000 samples per method and overlay the samples on target-density contours.

\paragraph{Target density definitions.}
We now state the precise density for each of the three two-dimensional synthetic targets used in Section~\ref{subsec:toy-experiments}.
All targets have $z\in\mathbb{R}^2$.

\subparagraph{X-shaped.}
The \texttt{x\_shaped} target is an equal-weight mixture of two zero-mean Gaussians whose covariance ellipses cross at the origin, producing an X-shaped density:
\begin{equation}\label{eq:x-shaped}
  p(z) = \tfrac{1}{2}\,\mathcal{N}(z;\,0,\,\Sigma_1)
        + \tfrac{1}{2}\,\mathcal{N}(z;\,0,\,\Sigma_2),
\end{equation}
where
\[
  \Sigma_1 = \begin{pmatrix} 2 & 1.8 \\ 1.8 & 2 \end{pmatrix},
  \qquad
  \Sigma_2 = \begin{pmatrix} 2 & -1.8 \\ -1.8 & 2 \end{pmatrix}.
\]
The high off-diagonal correlation ($\rho = \pm 0.9$) in each component creates two narrow ridges that intersect at the origin.  Approximating both ridges simultaneously requires the variational family to represent sharp, non-convex geometry.

\subparagraph{Student-t (uncorrelated marginals).}
The \texttt{student\_uc} target is a bivariate Student-$t$ distribution with $\nu=2$ degrees of freedom and a fixed scale matrix $A\in\mathbb{R}^{2\times 2}$ with $\det(A)=1$.  Its density is defined through the transformation $\tilde{z} = A^{-1}z$ with independent univariate Student-$t$ marginals:
\begin{equation}\label{eq:student-uc}
  p(z) \;=\; |\!\det A|^{-1}\,\prod_{i=1}^{2}
    \frac{\Gamma\!\bigl(\tfrac{\nu+1}{2}\bigr)}
         {\sqrt{\nu\pi}\;\Gamma\!\bigl(\tfrac{\nu}{2}\bigr)}
    \left(1 + \frac{\tilde{z}_i^{\,2}}{\nu}\right)^{\!-(\nu+1)/2},
  \qquad \nu = 2.
\end{equation}
The matrix $A$ is generated deterministically (seed~50) as $A = (B B^\top)^{-1}/\det\bigl((B B^\top)^{-1}\bigr)$, where $B_{ij}\sim\mathrm{Uniform}(-1,1)$.  The resulting density has polynomial tails ($\propto\|z\|^{-3}$ marginally), which stress-tests whether the variational approximation can allocate sufficient probability mass far from the mode.

\subparagraph{Eight Gaussians.}
The \texttt{8\_gaussians} target is an equally weighted mixture of eight isotropic Gaussians arranged uniformly on a circle of radius~$r$:
\begin{equation}\label{eq:8-gaussians}
  p(z) = \frac{1}{8}\sum_{k=1}^{8}
    \mathcal{N}\!\bigl(z;\,\mu_k,\,\sigma^2 I_2\bigr),
  \qquad
  \mu_k = r\bigl(\cos(k\pi/4),\;\sin(k\pi/4)\bigr),
\end{equation}
with $r=4$ and $\sigma=0.5$.  The well-separated modes (inter-mode distance $\approx 3.1$, versus component standard deviation $0.5$) make this a challenging multi-modal target: the variational family must discover and cover all eight modes without spurious bridging density between them.

\bigskip

Table~\ref{tab:toy-target-summary} collects the key parameters of each target.

\begin{table}[h]
\centering
\caption{Summary of two-dimensional synthetic target densities.}
\label{tab:toy-target-summary}
\small
\begin{tabular}{@{}llll@{}}
\toprule
\textbf{Target} & \textbf{Density form} & \textbf{Parameters} & \textbf{Key challenge} \\
\midrule
\texttt{x\_shaped}
  & $\frac{1}{2}\mathcal{N}(0,\Sigma_1) + \frac{1}{2}\mathcal{N}(0,\Sigma_2)$
  & $\Sigma_{1,2} = \bigl(\begin{smallmatrix}2 & \pm1.8\\ \pm1.8 & 2\end{smallmatrix}\bigr)$
  & Intersecting ridges \\[6pt]
\texttt{student\_uc}
  & $|A|^{-1}\prod_{i=1}^{2} t_\nu(\tilde{z}_i)$, \; $\tilde{z}=A^{-1}z$
  & $\nu=2$, \; $\det(A)=1$
  & Heavy polynomial tails \\[6pt]
\texttt{8\_gaussians}
  & $\frac{1}{8}\sum_{k=1}^{8}\mathcal{N}(\mu_k, \sigma^2 I)$
  & $r=4$, \; $\sigma=0.5$
  & Multi-modal coverage \\
\bottomrule
\end{tabular}
\end{table}

\paragraph{Training configuration details.}
SIVI, UIVI, AISIVI, and DIVI use 10,000 iterations, VI learning rate $10^{-3}$, batch size 128, and a step scheduler with decay factor 0.7 every 1,000 iterations (0.9 on \texttt{8\_gaussians}). Linear annealing over 5,000 steps is applied on all targets except \texttt{student\_uc}, where annealing is disabled for these methods.

\subparagraph{SIVI.}
SIVI uses 4,096 auxiliary samples from the mixing prior to approximate the semi-implicit marginal density.

\subparagraph{AISIVI.}
AISIVI uses a conditional RealNVP reverse model with 8 coupling layers and hidden width 12.  Before joint training, the flow is pre-trained on initial VI samples for 500 warmup epochs with learning rate $10^{-2}$ and batch size 8,192.  During main training, the flow uses 1,024 reverse samples, reverse learning rate $10^{-3}$, and one reverse update per VI step.

\subparagraph{DIVI.}
The proxy score MLP has hidden width 64 and three hidden layers.  The proxy is warm-started for 500 epochs with learning rate $10^{-3}$ and batch size 8,192.  During main training the proxy uses reverse batch size 2,048, reverse learning rate $10^{-3}$, and 5 updates per VI step.

\subparagraph{KSIVI.}
KSIVI is trained for 50,000 iterations with batch size 128, learning rate $10^{-3}$, 25,000-step linear annealing, and a step-decay schedule with factor 0.9 every 1,000 iterations, following \citet{cheng2024kernel}. We use a Gaussian kernel on \texttt{x\_shaped} and \texttt{8\_gaussians}, and an affine-invariant Riesz kernel on \texttt{student\_uc}. On \texttt{x\_shaped}, the conditional variances are initialized to the constant $1.313$; the mean network uses learning rate $10^{-3}$ and the variance head uses $10^{-5}$, with the same decay schedule. On \texttt{student\_uc}, the median kernel bandwidth is detached during backpropagation, and we apply log probability regularization with coefficient $0.05$ during annealing following \citet{cheng2024kernel}.

\paragraph{Quantitative results.}
Table~\ref{tab:toy-method-grid} presents the quantitative results across all five methods. DIVI obtains the lowest mean on four of the six quality metrics and the second-lowest on the remaining two, while SIVI is strongest on the Student UC and eight-Gaussian KL-style metrics. DIVI also has the lowest mean end-to-end wall-clock time. The Wasserstein-2 evaluation follows the protocol in Appendix~\ref{app:metric-details}. We are observing divergence on \texttt{8_gaussians} for AISIVI, as a result of drifting support. The same divergence behavior has been observed in the official code of AISIVI, indicating a potential weakness in training stability.

\input{tables/toy_method_grid.tex}

\subsection{Conditioned diffusion process}
\label{app:diffusion-details}

The diffusion posterior is the 100-dimensional discretized path posterior described in Section~\ref{subsec:langevin-experiments}.  Reference samples are produced by SGLD with 1,000 particles, 100,000 iterations, and step size $10^{-4}$.  Trace plots show posterior mean paths and marginal 95\% posterior intervals against the true latent path and noisy observations.

SIVI, UIVI, AISIVI, and DIVI are trained for 10,000 VI iterations with learning rate $10^{-3}$, batch size 128, step scheduler (decay 0.7 every 1,000 steps), and 5,000-step linear annealing.  KSIVI is trained for 100,000 iterations.

\paragraph{AISIVI.}
AISIVI uses a conditional RealNVP with 16 coupling layers and hidden width 256.  Before joint training, the reverse flow is pre-trained on samples from the initial variational distribution for 2,000 warmup epochs with learning rate $2 \times 10^{-4}$ and batch size 128.  During main training, the flow uses 256 reverse samples, reverse learning rate $2 \times 10^{-4}$, and one reverse update per VI step.

\paragraph{DIVI.}
The proxy score MLP has hidden width 128 and three hidden layers.  The proxy is warm-started for 2,000 epochs with learning rate $10^{-3}$ and batch size 8,192.  During main training the proxy uses reverse batch size 2,048, reverse learning rate $10^{-3}$, and 5 updates per VI step.

\paragraph{KSIVI.}
KSIVI follows the default schedule from \citet{cheng2024kernel}: 100,000 iterations, 50,000-step linear annealing, Gaussian kernel, and a step-decay schedule with factor 0.9 every 10,000 iterations.

\paragraph{Evaluation.}
We evaluate using the expected log marginal likelihood defined in Appendix~\ref{app:metric-details}, with 1,000 reference samples from the SGLD chain and 100,000 model samples from each trained variational distribution.

\subsection{Bayesian neural-network regression}
\label{app:bnn-details}

Each BNN target uses a one-hidden-layer ReLU network with 50 hidden units.  The input dimensions are 13 for Boston, 8 for Concrete, 4 for Power, 9 for Protein, 11 for Wine Red, and 6 for Yacht. The datasets can be found on the UCI Machine Learning website \footnote{\url{https://archive.ics.uci.edu/}}, and all follow CC BY 4.0 License to the best of our knowledge. Including weights, biases, and two precision variables, the corresponding posterior dimensions are 751, 501, 301, 551, 651, and 401.  The datasets contain 506, 1,030, 9,568, 45,730, 1,599, and 308 observations respectively.  Data preparation uses a fixed random seed, a 90/10 train-test split, input standardization based on the training split, and target standardization for training labels.  Test labels are evaluated on their original scale.

\paragraph{Common BNN protocol.}
All methods use batch size 128, gradient clipping at norm 1, and a 100-step MSE pretraining stage that initializes the variational sampler toward predictive weights before posterior training begins.  For evaluation, we apply exponential moving averages of the VI parameters with decay 0.999, following the stabilized BNN evaluation protocol used by KSIVI and SIVI-SM \citep{cheng2024kernel,yu2023semi}.  The VI learning rate is selected separately for each method and target from $\{10^{-3}, 10^{-4}, 10^{-5}\}$ by cross-validation, since high-learning-rate divergence depends on both the method and the target posterior.

\paragraph{AISIVI.}
AISIVI uses a conditional RealNVP with 16 coupling layers and hidden width 256.  The reverse flow is pre-trained on initial VI samples for 500 warmup epochs with learning rate $5 \times 10^{-5}$ and batch size 100.  During main training, the flow uses 256 reverse samples, reverse learning rate $5 \times 10^{-5}$, one reverse update per VI step, and 10,000-step linear annealing.

\paragraph{DIVI.}
The proxy score MLP has hidden width 512 and four hidden layers.  The proxy is warm-started for 5,000 epochs with learning rate $2 \times 10^{-3}$ and batch size 2,048.  During main training the proxy uses 2 updates per VI step.  Linear annealing over 10,000 steps is applied.

\paragraph{KSIVI.}
KSIVI is trained for 100,000 iterations without annealing, using a Gaussian kernel and log-likelihood regularization term with weight 1, with a step-decay schedule of factor 0.9 every 3,000 iterations, following \citet{cheng2024kernel}.

\paragraph{Evaluation.}
BNN evaluation draws 2,000 posterior samples and reports test RMSE and NLL as described in Appendix~\ref{app:metric-details}.

\subsection{Direct score-estimation evaluation}
\label{app:score-estimation}

We evaluate the marginal-score error on fixed variational models obtained after 2,000, 4,000, 6,000, 8,000, and 10,000 DIVI training steps on the X-shaped target. For each model, we generate $N=1{,}024$ evaluation pairs by sampling $\epsilon_i^{(0)}\sim\mathcal{N}(0,I)$ and then $z_i\sim q_\phi(z\mid\epsilon_i^{(0)})$. Following the reverse-conditional score identity used in UIVI~\citep{titsias2019unbiased},
\begin{equation}
  \nabla_z\log q_\phi(z)
  = \mathbb{E}_{q_\phi(\epsilon\mid z)}
    \left[\nabla_z\log q_\phi(z\mid\epsilon)\right],
  \qquad
  q_\phi(\epsilon\mid z)\propto q(\epsilon)q_\phi(z\mid\epsilon),
\end{equation}
we construct a reference score $s_{\mathrm{HMC}}(z_i)$ by sampling the reverse conditional. For each $z_i$, we run 20 HMC chains initialized at the generating $\epsilon_i^{(0)}$, with 1,000 burn-in steps and 5,000 retained samples per chain. We pool the retained samples and average their conditional Gaussian scores. The reported HMC diagnostic is $\hat{R}=1.0065$. The empirical squared score-estimation error is
\begin{equation}
  \widehat{\mathcal{E}}^2
  = \frac{1}{N}\sum_{i=1}^{N}
    \left\|\widehat{s}_{\mathrm{method}}(z_i)-s_{\mathrm{HMC}}(z_i)\right\|_2^2.
\end{equation}

Every estimator is applied to the same fixed DIVI variational model at each checkpoint. SIVI uses 4,096 mixing-prior samples; UIVI uses its native posterior-HMC estimator with five burn-in transitions and five auxiliary samples; AISIVI uses 1,024 importance samples from a normalizing-flow proposal optimized for 10,000 steps separately for each fixed variational distribution; and DIVI uses its checkpointed denoising score network. Table~\ref{tab:score-estimation-error} reports score error means and standard errors across five seeds. Score-evaluation time is measured for a batch of 128 over 100 synchronized calls after fitting the score network or proposal.

\subsection{Discrepancy from the target score}
\label{app:target-score-discrepancy}

We also evaluate the squared $L^2$ discrepancy between each method's estimated variational score and the analytic score of the X-shaped target throughout training. At each checkpoint, we use the variational model obtained by that method and evaluate
\begin{equation}
  D_{\mathrm{target}}
  = \mathbb{E}_{z\sim q_\phi}
    \left\|\widehat{s}_{\mathrm{method}}(z)-\nabla_z\log p(z)\right\|_2^2.
\end{equation}
This diagnostic measures the joint effect of variational approximation and score estimation along each method's own training trajectory. Table~\ref{tab:target-score-discrepancy} reports the mean and standard error across five seeds. DIVI obtains the lowest mean target-score discrepancy at every reported checkpoint.

\begin{table}[!htbp]
\centering
\caption{Squared $L^2$ discrepancy from the analytic X-shaped target score along each method's training trajectory (mean $\pm$ standard error over five seeds). Bold indicates the lowest mean.}
\label{tab:target-score-discrepancy}
\small
\setlength{\tabcolsep}{3pt}
\begin{adjustbox}{max width=\linewidth}
\begin{tabular}{lccccc}
\toprule
Method & 2K steps & 4K steps & 6K steps & 8K steps & 10K steps \\
\midrule
SIVI & $4.479\pm1.186$ & $0.546\pm0.306$ & $0.298\pm0.248$ & $0.270\pm0.221$ & $0.190\pm0.152$ \\
UIVI & $5.062\pm0.226$ & $3.265\pm0.676$ & $3.023\pm0.276$ & $3.052\pm0.321$ & $3.040\pm0.356$ \\
AISIVI & $1704.846\pm1662.003$ & $282.647\pm161.966$ & $581.071\pm429.919$ & $100.476\pm27.904$ & $88.381\pm38.811$ \\
DIVI & $\mathbf{3.694\pm0.200}$ & $\mathbf{0.411\pm0.049}$ & $\mathbf{0.148\pm0.028}$ & $\mathbf{0.142\pm0.038}$ & $\mathbf{0.079\pm0.017}$ \\
\bottomrule
\end{tabular}
\end{adjustbox}
\end{table}

\subsection{Effect and selection of score-update steps and batch size}
\label{app:reverse-updates}

The number of inner score updates per outer VI update, $L_\psi$ (reverse steps), and the score batch size, $B_\psi$ (reverse batch size), control the accuracy and computational cost of score tracking. DIVI warm-starts $\psi$ from the preceding iteration, so each inner loop tracks the score of the current variational distribution. The main experiments use two inner updates for BNN regression and five for the toy and conditioned diffusion targets. The appropriate choice depends on dimension, network capacity, batch size, and how rapidly $q_\phi$ changes. A practical selection rule starts with a larger update count, such as $L_\psi=10$, and halves it while monitoring a distributional metric such as the ELBO or Wasserstein-2 distance. The selected value balances distributional accuracy and computation. For batch-size selection, $B_\psi=2{,}048$ provides a starting point; larger batches can be evaluated using the same distributional metric within the available memory budget.

Tables~\ref{tab:reverse-steps} and~\ref{tab:reverse-batch-size} give ablations on the X-shaped target over five seeds, with the remaining training settings fixed. The first varies $L_\psi$ with $B_\psi=2{,}048$; the second varies $B_\psi$ with $L_\psi=5$. Increasing the update count improves both metrics and increases runtime. Larger score batches improve both metrics with similar measured runtimes over the tested range. The main toy experiments use $L_\psi=5$ and $B_\psi=2{,}048$.

\begin{table}[!htbp]
\centering
\caption{Effect of reverse steps on the X-shaped target with reverse batch size 2,048 (mean $\pm$ standard error over five seeds). Bold indicates the lowest mean quality metric.}
\label{tab:reverse-steps}
\small
\begin{tabular}{lccc}
\toprule
Reverse steps & $D_{\mathrm{KL}}\downarrow$ & W2 $\downarrow$ & Runtime (s) \\
\midrule
1 & $0.2083\pm0.1503$ & $0.2517\pm0.0526$ & $171.6\pm1.1$ \\
2 & $0.02253\pm0.00725$ & $0.2402\pm0.0375$ & $216.5\pm3.4$ \\
5 (default) & $0.004619\pm0.001524$ & $0.1127\pm0.0279$ & $352.8\pm4.8$ \\
10 & $\mathbf{0.002708\pm0.000824}$ & $\mathbf{0.07333\pm0.00553}$ & $571.7\pm7.2$ \\
\bottomrule
\end{tabular}
\end{table}

\begin{table}[!htbp]
\centering
\caption{Effect of reverse batch size on the X-shaped target with five reverse steps (mean $\pm$ standard error over five seeds). Bold indicates the lowest mean quality metric.}
\label{tab:reverse-batch-size}
\small
\begin{tabular}{lccc}
\toprule
Reverse batch size & $D_{\mathrm{KL}}\downarrow$ & W2 $\downarrow$ & Runtime (s) \\
\midrule
512 & $0.01058\pm0.00520$ & $0.1592\pm0.0387$ & $339.6\pm3.6$ \\
1,024 & $0.008070\pm0.001921$ & $0.1480\pm0.0249$ & $339.4\pm5.5$ \\
2,048 (default) & $0.004619\pm0.001524$ & $0.1127\pm0.0279$ & $352.8\pm4.8$ \\
4,096 & $0.004079\pm0.001026$ & $0.07773\pm0.01072$ & $351.1\pm4.1$ \\
8,192 & $\mathbf{0.001978\pm0.000539}$ & $\mathbf{0.06823\pm0.00606}$ & $344.7\pm5.2$ \\
\bottomrule
\end{tabular}
\end{table}

\subsection{Comparison with a RealNVP variational family}
\label{app:flow-comparison}

Normalizing-flow VI uses invertible transformations with tractable densities and scores. Semi-implicit VI represents complex multimodal structure through an implicit mixture. DIVI learns the marginal score of this family. As a concrete comparison, we train a four-layer RealNVP~\citep{dinh2016density} for 10,000 optimization steps with batch size 128 on the eight-Gaussian target and aggregate the results across five seeds. Table~\ref{tab:realnvp-comparison} reports ELBO, Wasserstein-2 distance, and parameter count. DIVI achieves a higher mean ELBO and a lower mean Wasserstein-2 distance in this benchmark with fewer parameters.

\begin{table}[!htbp]
\centering
\caption{DIVI and a four-layer RealNVP variational family on the eight-Gaussian target (mean $\pm$ standard error over five seeds). Bold indicates the best mean or lowest parameter count.}
\label{tab:realnvp-comparison}
\small
\begin{tabular}{lccc}
\toprule
Method & ELBO $\uparrow$ & W2 $\downarrow$ & Parameter count \\
\midrule
DIVI & $\mathbf{-0.355\pm0.082}$ & $\mathbf{0.327\pm0.060}$ & \textbf{26,054} \\
RealNVP & $-0.374\pm0.006$ & $0.402\pm0.053$ & 137,236 \\
\bottomrule
\end{tabular}
\end{table}



\section{Numerical Validations}
\label{app: numerical-experiments}

In this section, we provide empirical diagnostics supporting Assumption~\ref{assump:bounded_score} and Assumption~\ref{assump:bounded_reparam} through numerical experiments along the observed training trajectories. 

For Assumption~\ref{assump:bounded_score}, we provide empirical diagnostics by tracking the fourth moment of scores, both for the reverse model and for the target score. We plot the mean value of the fourth moment of norms of the reverse model score $\mathbb{E}_{z\sim q_\phi} \lVert s_\psi(z) \rVert^4$ as the solid line, and the minimum and maximum values as the shaded regions in Figure~\ref{fig:score_q_4th_moment}, with each of the values calculated across 10 independent seeds. The same plot for fourth moment of the target score $\mathbb{E}_{z\sim q_\phi} \lVert \nabla_z \log p(\mathcal{D}, z) \rVert^4$ is shown in Figure~\ref{fig:score_p_4th_moment}. The figures provide empirical evidence that there exists a uniform bound on the $4^{\text{th}}$ moment of the norm of the scores over the training trajectories, consistent with  Assumption~\ref{assump:bounded_score}.


\begin{figure}[t]
    \centering
    \includegraphics[width=0.8\linewidth]{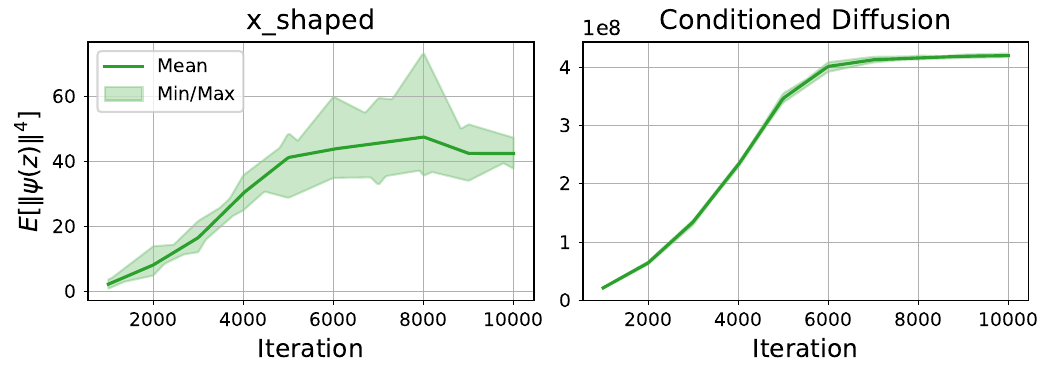}
    \caption{Fourth moment of the norm of the approximated score $\mathbb{E}_{z\sim q_\phi} \lVert s_\psi(z) \rVert^4$ throughout DIVI training for \texttt{x\_shaped} and the conditioned diffusion benchmark, aggregated over 10 seeds.}
    \label{fig:score_q_4th_moment}
\end{figure}

\begin{figure}[t]
    \centering
    \includegraphics[width=0.8\linewidth]{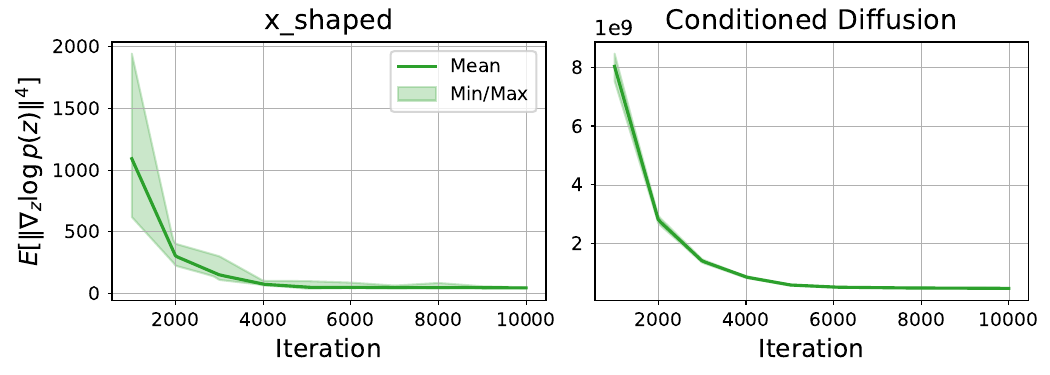}
    \caption{Fourth moment of the norm of the target score $\mathbb{E}_{z\sim q_\phi} \lVert \nabla_z \log p(\mathcal{D}, z) \rVert^4$ throughout DIVI training for \texttt{x\_shaped} and the conditioned diffusion benchmark, aggregated over 10 seeds.}
    \label{fig:score_p_4th_moment}
\end{figure}

For Assumption~\ref{assump:bounded_reparam}, we report the calculated $\max \big\{ \mathbb{E}_{\epsilon} \lVert \nabla_\phi \sigma_\phi(\epsilon)\rVert_2^4, \mathbb{E}_\epsilon \lVert \nabla_\phi \mu_\phi(\epsilon) \rVert_2^4\big\}$ for each checkpoint respectively throughout the training process, and aggregate across 10 independent seeds. We track the metrics for the toy \texttt{x_shaped} and Conditioned Diffusion target distribution. The results are shown in Figure~\ref{fig:bounded_m_eps}, where the solid line indicates the mean while the shaded region demonstrates the minimum and maximum values. These diagnostics suggest that we can obtain a uniform upper bound for the expected $\ell_2$ norm's fourth moment of the VI networks' gradient, justifying Assumption~\ref{assump:bounded_reparam}.

\begin{figure}[t]
    \centering
    \includegraphics[width=0.8\linewidth]{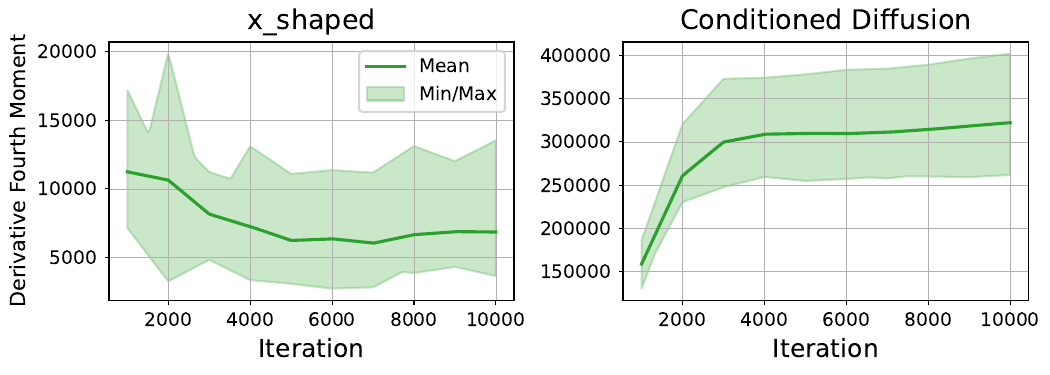}
    \caption{The expected $\ell_2$ norm's fourth moment of the VI model gradient throughout the DIVI training process for \texttt{x_shaped} and Conditioned Diffusion experiments, aggregated across 10 different seeds. This supports Assumption~\ref{assump:bounded_reparam} that the VI model gradients are bounded. }
    \label{fig:bounded_m_eps}
\end{figure}

We also provide empirical evidence to the alternative version of Assumption~\ref{assump:bounded_score}, as present in Appendix~\ref{app:linear-growth}. In Appendix~\ref{app:linear-growth} we proved that Assumption~\ref{assump:bounded_score} can be derived from another linear growth condition $\max\{\lVert \nabla_z \log p(\mathcal{D},z) \rVert, \lVert s_\psi(z)\rVert \} \leq \lambda (\lVert z \rVert + 1)$ together with a bounded fourth-moment condition $\max\{ \mathbb{E}_\epsilon \lVert \mu_\phi(\epsilon) \rVert^4, \mathbb{E}_\epsilon \lVert \sigma_\phi(\epsilon) \rVert^4 \} \leq \zeta $, both of which we assume to hold uniformly over $\phi$. As empirical evidence, we demonstrate for three checkpoints taken at 2000, 5000 and 10000 iterations, that $\log \max\{\lVert \nabla_z \log p(\mathcal{D},z) \rVert, \lVert s_\psi(z)\rVert \} - \log (\lVert z \rVert + 1)$ is bounded in Figure~\ref{fig:score_norm_linearity}. The figure demonstrates different values for samples of $z$ drawn with uniform distribution in $\lVert z \rVert$.

We also plot the $4^{\text{th}}$ moment of the VI model checkpoints $\max\{ \mathbb{E}_\epsilon \lVert \mu_\phi(\epsilon) \rVert^4, \mathbb{E}_\epsilon \lVert \sigma_\phi(\epsilon) \rVert^4 \}$ using 10 different seeds and plot them along the 10 checkpoints we obtained during the DIVI training regime in Figure~\ref{fig:vi_fourth_moment}. The solid line indicates the 10-seed-average of the fourth moments, while the shaded regions indicate the minimum and maximum of recorded values. These plots together demonstrate empirical evidence to support the alternate version of Assumption~\ref{assump:bounded_score} in Appendix~\ref{app:linear-growth}.

\begin{figure}[t]
    \centering
    \includegraphics[width=0.8\linewidth]{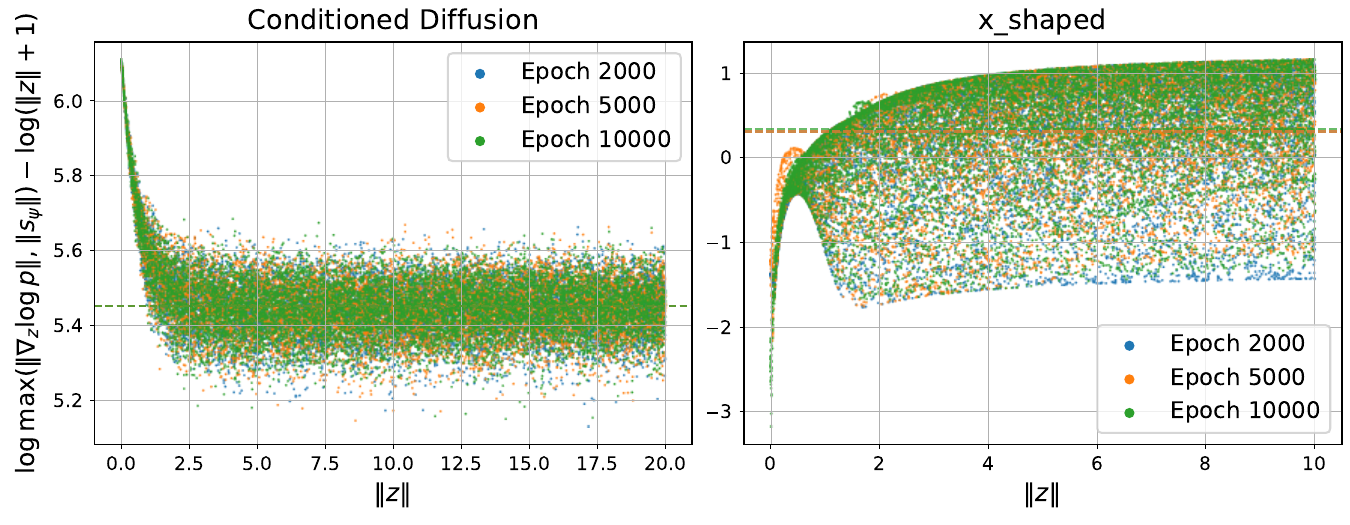}
    \caption{$\log \max\{\lVert \nabla_z \log p(\mathcal{D},z) \rVert, \lVert s_\psi(z)\rVert \} - \log (\lVert z \rVert + 1)$ over different $z$'s, $\phi$'s and $\psi$'s throughout DIVI training for \texttt{x\_shaped} and the conditioned diffusion benchmark.}
    \label{fig:score_norm_linearity}
\end{figure}

\begin{figure}[t]
    \centering
    \includegraphics[width=0.8\linewidth]{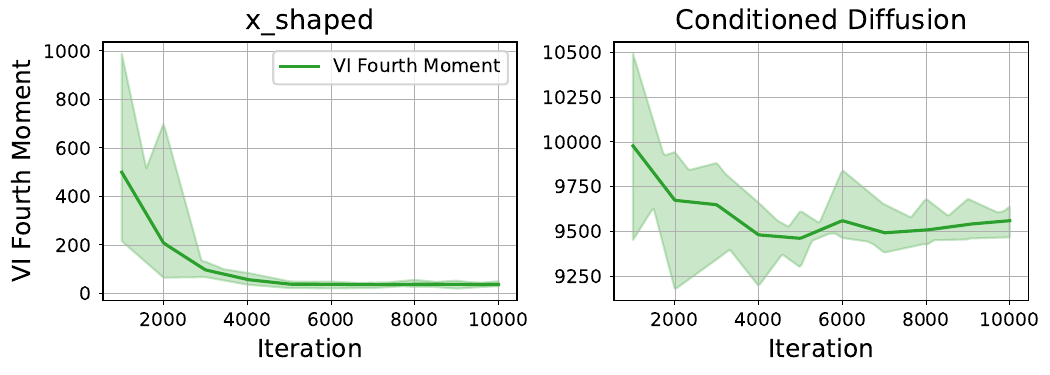}
    \caption{Fourth moment of the VI model reparameterization norm $\max\{ \mathbb{E}_\epsilon \lVert \mu_\phi(\epsilon) \rVert^4, \mathbb{E}_\epsilon \lVert \sigma_\phi(\epsilon) \rVert^4 \}$ throughout DIVI training for \texttt{x\_shaped} and the conditioned diffusion benchmark, aggregated over 10 seeds.}
    \label{fig:vi_fourth_moment}
\end{figure}

%% file: tables/toy_method_grid.tex
\begin{table}[t]
\centering
\caption{Toy target final evaluation metrics (mean $\pm$ standard error). Wall-clock time is averaged across the three targets; bold indicates the best mean.}
\label{tab:toy-method-grid}
\small
\setlength{\tabcolsep}{3pt}
\begin{adjustbox}{max width=\linewidth}
\begin{tabular}{llccccc}
\toprule
Target & Metric & SIVI & KSIVI & AISIVI & UIVI & DIVI \\
\midrule
\multirow{2}{*}{x\_shaped} & $D_{\mathrm{KL}}$ & 0.012 $\pm$ {\footnotesize 0.01} & 0.054 $\pm$ {\footnotesize 0.01} & 0.070 $\pm$ {\footnotesize 0.02} & 0.038 $\pm$ {\footnotesize 0.02} & \textbf{0.009} $\pm$ {\footnotesize \textbf{0.00}} \\
 & W2 & 0.075 $\pm$ {\footnotesize 0.03} & 0.222 $\pm$ {\footnotesize 0.03} & 0.191 $\pm$ {\footnotesize 0.03} & 0.152 $\pm$ {\footnotesize 0.04} & \textbf{0.074} $\pm$ {\footnotesize \textbf{0.01}} \\
\addlinespace[2pt]
\multirow{2}{*}{student\_uc} & $D_{\mathrm{KL}}$ & \textbf{2.676} $\pm$ {\footnotesize \textbf{0.00}} & 4.512 $\pm$ {\footnotesize 1.79} & 2.744 $\pm$ {\footnotesize 0.06} & 2.682 $\pm$ {\footnotesize 0.00} & 2.677 $\pm$ {\footnotesize 0.00} \\
 & W2 $|x|<8$ & 0.124 $\pm$ {\footnotesize 0.01} & 0.749 $\pm$ {\footnotesize 0.63} & 0.166 $\pm$ {\footnotesize 0.05} & 0.145 $\pm$ {\footnotesize 0.00} & \textbf{0.086} $\pm$ {\footnotesize \textbf{0.01}} \\
\addlinespace[2pt]
\multirow{2}{*}{8\_gaussians} & $D_{\mathrm{KL}}$ & \textbf{0.236} $\pm$ {\footnotesize \textbf{0.03}} & 0.830 $\pm$ {\footnotesize 0.08} & 1127.982 $\pm$ {\footnotesize 793.92} & 0.801 $\pm$ {\footnotesize 0.01} & 0.355 $\pm$ {\footnotesize 0.08} \\
 & W2 $|x|<6$ & 0.348 $\pm$ {\footnotesize 0.08} & 0.569 $\pm$ {\footnotesize 0.21} & 0.789 $\pm$ {\footnotesize 0.24} & 0.445 $\pm$ {\footnotesize 0.06} & \textbf{0.311} $\pm$ {\footnotesize \textbf{0.07}} \\
\midrule
\multicolumn{2}{l}{Wall-clock time (s)} & 266 & 377 & 387 & 803 & \textbf{244} \\
\bottomrule
\end{tabular}
\end{adjustbox}
\end{table}